\documentclass[letterpaper, 10 pt, conference]{ieeeconf}  % Comment this line out if you need a4paper
\IEEEoverridecommandlockouts                              % This command is only needed if 
\usepackage{balance}

\usepackage{caption}
\usepackage{benktools}
\usepackage{tikz}
\usetikzlibrary{external}
\usetikzlibrary{positioning,shapes,shadows,arrows, calc, fit, backgrounds, arrows, positioning}
\usepackage[switch, columnwise]{lineno}
\usepackage{amsmath,graphics}
\usepackage{amssymb} 
\usepackage{nicefrac, multirow}
\usepackage{algpseudocode}
\usepackage[linesnumbered]{algorithm2e}
\usepackage{tikz}
\usepackage{bigstrut}
\usepackage{pgfplots}
\pgfplotsset{compat=newest}
\usepackage{subfig}
\usepackage{textcomp}
\usetikzlibrary{backgrounds}
\numberwithin{equation}{section} 

\tikzstyle{every node}=[font=\small]

\usepackage{tabularx}
\usepackage{pbox}
\usepackage{ragged2e}
\usepackage{multirow}
\usepackage{floatrow}
\newfloatcommand{capbtabbox}{table}[][\FBwidth]

\newcolumntype{L}[1]{>{\RaggedRight\hspace{0pt}}p{#1}}
\newcolumntype{R}[1]{>{\RaggedLeft\hspace{0pt}}p{#1}}

\newcommand{\bR}{\mathbb{R}}

\renewcommand{\phi}{\varphi}

\renewcommand{\geq}{\geqslant}
\renewcommand{\leq}{\leqslant}

\newcommand{\furlp}[1]{\colorbox{blue!10}{\href{run:/home/fulong/academia/library/papers/#1.pdf}{D}}}
\newcommand{\furlb}[1]{\colorbox{blue!10}{\href{run:/home/fulong/academia/library/books/#1.pdf}{D}}}

\newcommand{\bx}{\mathbf{x}}

\newcommand{\bu}{\mathbf{u}}

\newcommand{\by}{\mathbf{y}}

\newcommand{\bp}{\mathbf{p}}

\newcommand{\bv}{\mathbf{v}}

\newcommand{\be}{\mathbf{e}}

\newcommand{\bzero}{\mathbf{0}}

\newcommand{\bw}{\mathbf{w}}

\newcommand{\mL}{\mathcal{L}}
\newcommand{\mN}{\mathcal{N}}

\newcommand{\mO}{\mathcal{O}}

\newcommand{\mC}{\mathcal{C}}

\newcommand{\mU}{\mathcal{U}}
\newcommand{\mD}{\mathcal{D}}
\newcommand{\mF}{\mathcal{F}}
\newcommand{\mS}{\mathcal{S}}

\newcommand{\mM}{\mathcal{M}}

\newcommand{\specialcell}[2][c]{ \begin{tabular}[#1]{@{}c@{}}#2\end{tabular}}

\newcommand{\myargmin}[1] {\underset{#1}{\text{argmin }}}

\newtheorem{proposition}{Proposition}
\newtheorem{definition}{Definition}

\begin{document}
\title{\LARGE \bf Dex-Net 3.0:  Computing Robust Vacuum Suction Grasp Targets in Point Clouds using a New Analytic Model and Deep Learning
	\vspace{-2ex}}
\author{Jeffrey Mahler$^1$, Matthew Matl$^1$, Xinyu Liu$^1$, Albert Li$^1$, David Gealy$^1$, Ken Goldberg$^{1,2}$
%\vspace{-2ex}
\thanks{{\small $^1$ Dept. of Electrical Engineering and Computer Science;}}%
\thanks{{\small $^2$ Dept. of Industrial Operations and Engineering Research; }}%
%\thanks{$^{1-2}$ University of California, Berkeley, USA}%
\thanks{{\small AUTOLAB University of California, Berkeley, USA {\tt\small \{jmahler, mmatl, xinyuliu, alberthli, dgealy, goldberg\}@berkeley.edu}} }
}
\maketitle

\begin{abstract}
Vacuum-based end effectors are widely used in industry and are often preferred over parallel-jaw and multifinger grippers due to their ability to lift objects with a single point of contact.
Suction grasp planners often target planar surfaces on point clouds near the estimated centroid of an object.
In this paper, we propose a compliant suction contact model that computes the quality of the seal between the suction cup and local target surface and a measure of the ability of the suction grasp to resist an external gravity wrench.
To characterize grasps, we estimate robustness to perturbations in end-effector and object pose, material properties, and external wrenches.
We analyze grasps across 1,500 3D object models to generate Dex-Net 3.0, a dataset of 2.8 million point clouds, suction grasps, and grasp robustness labels.
We use Dex-Net 3.0 to train a Grasp Quality Convolutional Neural Network (GQ-CNN) to classify robust suction targets in point clouds containing a single object.
We evaluate the resulting system in 350 physical trials on an ABB YuMi fitted with a pneumatic suction gripper.
When evaluated on novel objects that we categorize as Basic (prismatic or cylindrical), Typical (more complex geometry), and Adversarial (with few available suction-grasp points) Dex-Net 3.0 achieves success rates of 98$\%$, 82$\%$, and 58$\%$ respectively, improving to 81$\%$ in the latter case when the training set includes only adversarial objects.
%When the object shape, pose, and mass properties are known, the model achieves 99$\%$ precision on a dataset of objects with Adversarial geometry such as sharply curved surfaces.
%When those properties are not known, our GQ-CNN grasp planner achieves over 97$\%$ precision a dataset of 50 Basic and Typical objects.
Code, datasets, and supplemental material can be found at http://berkeleyautomation.github.io/dex-net.
\end{abstract}

\section{Introduction}
\seclabel{introduction}

Suction grasping is widely-used for pick-and-place tasks in industry and warehouse order fulfillment.
As shown in the Amazon Picking Challenge, suction has an advantage over parallel-jaw or multifinger grasping due to its ability to reach into narrow spaces and pick up objects with a single point of contact.
However, while a substantial body of research exists on parallel-jaw and multifinger grasp planning~\cite{bohg2014data}, comparatively little research has been published on planning suction grasps.

While grasp planning searches for gripper configurations that maximize a quality metric derived from mechanical wrench space analysis~\cite{murray1994mathematical}, human labels~\cite{saxena2008robotic}, or self-supervised labels~\cite{levine2016learning},
suction grasps are often planned directly on point clouds using heuristics such as grasping near the object centroid~\cite{hernandez2016team} or at the center of planar surfaces~\cite{correll2016analysis, domae2014fast}.
These heuristics work well for prismatic objects such as boxes and cylinders but may fail on objects with non-planar surfaces near the object centroid, which is common for industrial parts and household objects such as staplers or children's toys.
Analytic models of suction cups for grasp planning exist, but they typically assume that a vacuum seal has already been formed and that the state (e.g. shape and pose) of the object is perfectly known~\cite{bahr1996design, kolluru1998modeling, mantriota2007theoretical}.
In practice a robot may need to form seals on non-planar surfaces while being robust to external wrenches (e.g. gravity and disturbances), sensor noise, control imprecision, and calibration errors, which are significant factors when planning grasps from point clouds.

\begin{figure}[t!]
\centering
\includegraphics[width=\textwidth]{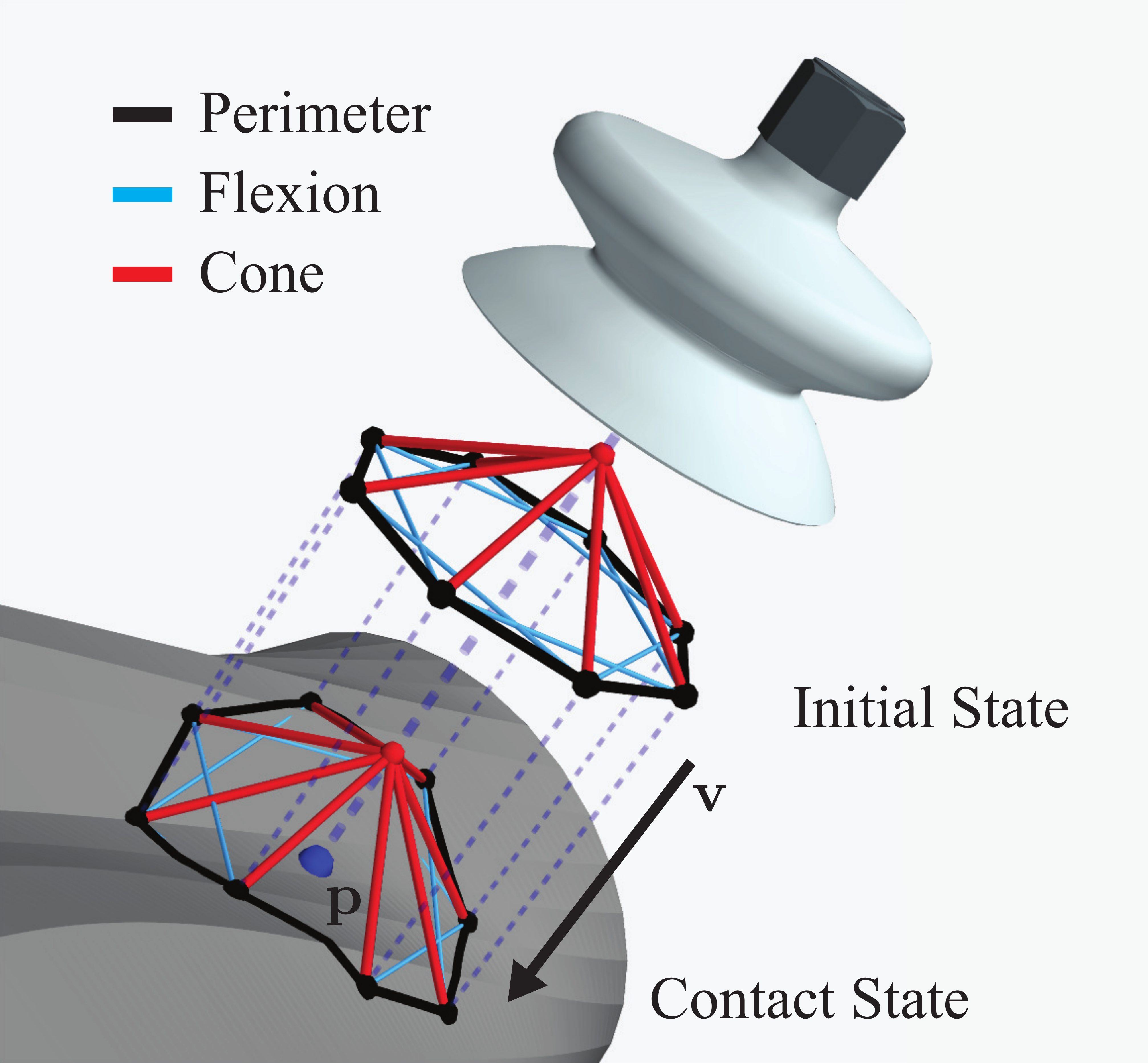}
\caption{The quasi-static spring model, $C$, used for determining when seal formation is feasible. The model contains three types of springs -- perimeter, flexion, and cone springs. An initial state for $C$ is chosen given a target point $\mathbf{p}$ and an approach direction $\mathbf{v}$. Then, a contact state for $C$ is computed so that $C$'s perimeter springs form a complete seal against object mesh $M$. Seal formation is deemed feasible if the energy required to maintain this contact state is sufficiently low in every spring.
\vspace*{-20pt}}
\figlabel{spring-approach}
\end{figure}

We propose a novel compliant suction contact model for rigid, non-porous objects that consists of two components: (1) a test for whether a seal can be formed between a suction cup and a target object surface and (2) an analysis of the ability of the suction contact to resist external wrenches.
We use the model to evaluate grasp robustness by analyzing seal formation and wrench resistance under perturbations in object pose, suction tip pose, material properties, and disturbing wrenches using Monte-Carlo sampling similar to that in the Dexterity Network (Dex-Net) 1.0~\cite{mahler2016dexnet}.

This paper makes four contributions:
\begin{enumerate}
	\item A compliant suction contact model that quantifies seal formation using a quasi-static spring system and the ability to resist external wrenches (e.g.  gravity) using a contact wrench basis derived from the ring of contact between the cup and object surface.
	\item Robust wrench resistance: a robust version of the above model under random disturbing wrenches and perturbations in object pose, gripper pose, and friction.
    \item Dex-Net 3.0, a dataset of 2.8 million synthetic point clouds annotated with suction grasps and grasp robustness labels generated by analyzing robust wrench resistance for approximately 375k grasps across 1,500 object models.
%     \item A GQ-CNN architecture that maps a 3D point cloud and the center, depth, and approach axis of a suction grasp to a scalar estimate of grasp robustness.
	\item Physical robot experiments measuring the precision of robust wrench resistance both with and without knowledge of the target object's shape and pose.
	%\item A taxonomy of object shapes for robot grasping experiments that differentiates 3 classes of objects (Basic, Typical, and Adversarial).
%	 \item A measure of grasp "precision" where success rate is a function of predicted confidence analogous to the receiver operating characteristic (ROC) curve in machine learning to differentiate between Type I and Type II grasp failures.
\end{enumerate}

We perform physical experiments using an ABB YuMi robot with a silicone suction cup tip to compare the precision of a GQ-CNN-based grasping policy trained on Dex-Net 3.0 with several heuristics such as targeting planar surfaces near object centroids.
We find that the method achieves success rates of 98$\%$, 82$\%$, and 58$\%$ on datasets of Basic (prismatic or cylindrical), Typical (more complex geometry), and Adversarial (with few available suction-grasp points), respectively.

\section{Related Work}
\seclabel{related-work}
%While rigid fingers can only apply forces directed inward on object surfaces, suction cups can pull an object using forces generated by an air pressure gradient between the inside and outside of the cup~\cite{marquardt1999introduction}.
%This makes suction advantageous when it is difficult or impossible to wrap fingers around an object due to size or the presence of obstacles.
End-effectors based on suction are widely used in industrial applications such as warehouse order fulfillment, handling limp materials such as fabric~\cite{kolluru1998modeling}, and robotics applications such as the Amazon Picking Challenge~\cite{correll2016analysis}, underwater manipulation~\cite{stuart2015suction}, or wall climbing~\cite{bahr1996design}.
Our method builds on models of deformable materials, analyses of the wrenches that suction cups can exert, and data-driven grasp planning.

\subsection{Suction Models}
Several models for the deformation of stiff rubber-like materials exist in the literature.
Provot et al.~\cite{provot1995deformation} modeled sheets of stiff cloth and rubber using a spring-mass system with several types of springs.
Hosseini et al.~\cite{ali2010review} provided a survey of more modern constitutive models of rubber that are often used in Finite Element Analysis (FEA) packages for more realistic physics simulations.
In order to rapidly evaluate whether a suction cup can form a seal against an object's surface, we model the cup as a quasi-static spring system with a topology similar to the one in~\cite{provot1995deformation} and estimate the deformation energy required to maintain a seal.

In addition, several models have been developed to check for static equilibrium assuming a seal between the suction cup and the object's surface.
Most models consider the suction cup to be a rigid object and model forces into the object along the surface normal, tangential forces due to surface friction, and pulling forces due to suction~\cite{kolluru1998modeling, stuart2015suction, valencia20173d}.
Bahr et al.~\cite{bahr1996design} augmented this model with the ability to resist moments about the center of suction to determine the amount of vacuum pressure necessary to keep a climbing robot attached to a vertical wall.
Mantriota~\cite{mantriota2007theoretical} modeled torsional friction due to a contact area between the cup and object similar to the soft finger contact model used in grasping~\cite{kao2008contact}.
%Xin et al.~\cite{xin2016development} experimentally characterized the relationship between the suction pressure distribution, air power consumption, and gasket height, finding that suction force is approximately proportional to air power.
Our model extends these methods by combining models of torsional friction~\cite{mantriota2007theoretical} and contact moments~\cite{bahr1996design} in a compliant model of the ring of contact between the cup and object.

\subsection{Grasp Planning}

% Intro
The goal of grasp planning is to select a configuration for an end-effector that enables a robot to perform a task via contact with an object while resisting external perturbations~\cite{bohg2014data}, which can be arbitrary~\cite{ferrari1992} or task-specific~\cite{krug2017grasp}.
A common approach is to select a configuration that maximizes a quality metric (or reward) based on wrench space analysis~\cite{murray1994mathematical}, robustness to perturbations~\cite{weisz2012pose}, or a model learned from human labels~\cite{kappler2015leveraging} or self-supervision~\cite{pinto2016supersizing}.

Several similar metrics exist for evaluating suction grasps.
One common approach is to evaluate whether or not a set of suction cups can lift an object by applying an upwards force~\cite{kolluru1998modeling, stuart2015suction, tsourveloudis2000suction, valencia20173d}.
%Tsourveloudis et al.~\cite{} developed a quality measure for suction grasps based on object porosity, wrenches due to gravity, and sheer forces due to motion during transport.
Domae et al.~\cite{domae2014fast} developed a geometric model to evaluate suction success by convolving target locations in images with a desired suction contact template to assess planarity.
Heuristics for planning suction grasps from point clouds have also been used extensively in the Amazon Robotics Challenge.
In 2015, Team RBO~\cite{eppner2016lessons} won by pushing objects from the top or side until suction was achieved, and Team MIT~\cite{yu2016summary} came in second place by suctioning on the centroid of objects with flat surfaces.
In 2016, Team Delft~\cite{hernandez2016team} won the challenge by approaching the estimated object centroid along the inward surface normal.
In 2017, Cartman~\cite{morrison2017cartman} won the challenge and ranked suction grasps according to heuristics such as maximizing distance to the segmented object boundary and MIT~\cite{zeng2017robotic} performed well using a neural network trained to predict grasp affordance maps from human labeled RGB-D point clouds.
In this work, we present a novel metric that evaluates whether a single suction cup can resist external wrenches under perturbations in object / gripper poses, friction coefficient, and disturbing wrenches.

%% Vision-based
This paper also extends empirical, data-driven approaches to grasp planning based on images and point clouds~\cite{bohg2014data}.
A popular approach is to use human labels of graspable regions in RGB-D images~\cite{lenz2015deep} or point clouds~\cite{kappler2015leveraging} to learn a grasp detector with computer vision techniques.
As labeling may be tedious for humans, an alternative is to automatically collect training data from a physical robot~\cite{levine2016learning, pinto2016supersizing}.
To reduce the time-cost of data collection, recent research has proposed to generate labels in simulation using physical models of contact~\cite{johns2016deep, kappler2015leveraging}.
Mahler et al.~\cite{mahler2017dex} demonstrated that a GQ-CNN trained on Dex-Net 2.0, a dataset of 6.7 million point clouds, grasps, and quality labels computed with robust quasi-static analysis, could be used to successfully plan parallel-jaw grasps across a wide variety of objects with 99$\%$ precision.
In this paper, we use a similar approach to generate a dataset of point clouds, grasps, and robustness labels for a suction-based end-effector.

\section{Problem Statement}
\seclabel{problem-statement}
Given a point cloud from a depth camera, our goal is to find a robust suction grasp (target point and approach direction) for a robot to lift an object in isolation on a planar worksurface and transport it to a receptacle.
We compute the suction grasp that maximizes the probability that the robot can hold the object under gravity and perturbations sampled from a distribution over sensor noise, control imprecision, and random disturbing wrenches.
%We model the probability that a candidate suction grasp (target point and approach direction) can lift and transport an object under perturbations due to sensor noise and disturbing wrenches and solve for the suction grasp that maximizes this probability.

\subsection{Assumptions}
\seclabel{assumptions}
Our stochastic model makes the following assumptions:
\begin{enumerate}
	\item Quasi-static physics (e.g. inertial terms are negligible) with Coulomb friction.
	\item Objects are rigid and made of non-porous material.
	\item Each object is singulated on a planar worksurface in a stable resting pose~\cite{goldberg1999part}.
	\item A single overhead depth sensor with known intrinsics, position, and orientation relative to the robot.
	\item A vacuum-based end-effector with known geometry and a single disc-shaped suction cup made of linear-elastic material.
\end{enumerate}

\subsection{Definitions}
\seclabel{description}

A robot observes a single-view {\bf point cloud} or depth image, $\by$, containing a singulated object.
The goal is to find the most robust suction {\bf grasp} $\bu$ that enables the robot to lift an object and transport it to a receptacle, where grasps are parametrized by a target point $\bp \in \bR^3$ and an approach direction $\bv \in \mS^2$.
Success is measured with a binary {\bf grasp reward function} $R$, where $R=1$ if the grasp $\bu$ successfully transports the object, and $R=0$ otherwise.

The robot may not be able to predict the success of suction grasps exactly from point clouds for several reasons.
First, the success metric depends on a {\bf state} $\bx$ describing the object's geometric, inertial, and material properties $\mO$ and the pose of the object relative to the camera, $T_o$, but the robot does not know the true state due to: (a) noise in the depth image and (b) occlusions due to the single viewpoint.
Second, the robot may not have perfect knowledge of external wrenches (forces and torques) on the object due to gravity or external disturbances.

This probabilistic relationship is described by an {\bf environment} consisting of a grasp success distribution modeling $\mathbb{P}(R\mid \bx, \bu)$, the ability of a grasp to resist random disturbing wrenches, and an observation model $p(\by \mid \bx)$.
This model induces a probability of success for each grasp conditioned on the robot's observation:
\begin{definition}
The {\it robustness} of a grasp $\bu$ given a point cloud $\by$ is the probability of grasp success under uncertainty in sensing, control, and disturbing wrenches: $Q(\by, \bu) = \mathbb{P} \left( R \mid \by, \bu \right)$.
\end{definition}
Our environment model is described in \secref{dataset} and further details are given in the supplemental file.

\subsection{Objective}
\seclabel{objective}
Our ultimate goal is to find a grasp that maximizes robustness given a point cloud, $\pi^*(\by) = \text{argmax}_{\bu \in \mC} Q(\by, \bu)$, where $\mC$ specifies constraints on the set of available grasps, such as collisions or kinematic feasibility.
We approximate $\pi^*$ by optimizing the weights $\theta$ of a deep Grasp Quality Convolutional Neural Network (GQ-CNN) $Q_{\theta}(\by, \bu)$ on a training dataset $\mD = \left\{ (\by_i, \bu_i, R_i) \right\}_{i=1}^N$ consisting of reward values, point clouds, and suction grasps sampled from our stochastic model of grasp success.
Our optimization objective is to find weights $\theta$ that minimize the cross-entropy loss $\mL$ over $\mD$:
\begin{align}
	\theta^* = \myargmin{\theta \in \Theta} \sum \limits_{i=1}^N \mL(R_i, Q_{\theta}(\by_i, \bu_i)).  \label{eq:obj}
\end{align}

\section{Compliant Suction Contact Model}
\seclabel{contact-model}

%We evaluate the success of suction grasps using a novel compliant contact model that analyzes whether:
To quantify grasp robustness, we develop a quasi-static spring model of the suction cup material and a model of contact wrenches that the suction cup can apply to the object through a ring of contact on the suction cup perimeter.
Under our model, the reward $R=1$ if:
\begin{enumerate}
	\item A seal is formed between the perimeter of the suction cup and the object surface.
	\item Given a seal, the suction cup is able to resist an external wrench on the object due to gravity and disturbances.
\end{enumerate}

\subsection{Seal Formation}
A suction cup can lift objects due to an air pressure differential induced across the membrane of the cup by a vacuum generator that forces the object into the cup.
If a gap exists between the perimeter of the cup and the object, then air flowing into the gap may reduce the differential and cause the grasp to fail.
Therefore, a tight seal between the cup and the target object is important for achieving a successful grasp.

\begin{figure}[t!]
\centering
\includegraphics[width=0.9\textwidth]{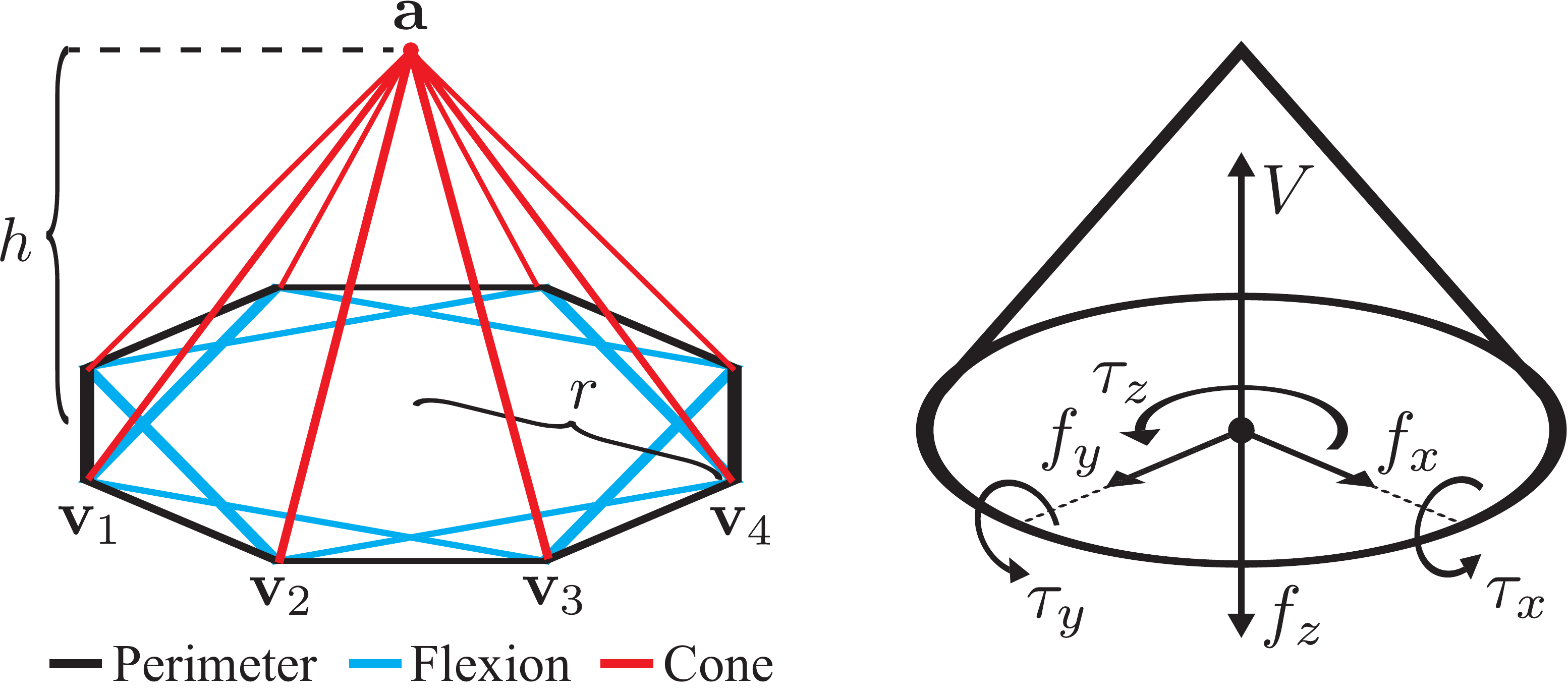}
\caption{Our compliant suction contact model. \textbf{(Left)} The quasi-static spring model used in seal formation computations. This suction cup is approximated by $n=8$ components. Here, $r$ is equal to the radius of the cup and $h$ is equal to the height of the cup. $\{\mathbf{v}_1, \ldots, \mathbf{v}_n\}$ are the base vertices and $\mathbf{a}$ is the apex. \textbf{(Right)} Wrench basis for the suction ring contact model. The contact exerts a constant pulling force on the object of magnitude $V$ and additionally can push or pull the object along the contact $z$ axis with force $f_z$. The suction cup material exerts a normal force $f_N = f_z + V$ on the object through a linear pressure distribution on the ring. This pressure distribution induces a friction limit surface bounding the set of possible frictional forces in the tangent plane $f_t = (f_x, f_y)$ and the torsional moment $\tau_z$, and also induces torques $\tau_x$ and $\tau_y$ about the contact $x$ and $y$ axes due to elastic restoring forces in the suction cup material.
\vspace*{-20pt}}
\figlabel{model}
\end{figure}

To determine when seal formation is possible, we model circular suction cups as a conical spring system $\mC$ parameterized by real numbers $(n, r, h)$, where $n$ is the numer of vertices along the contact ring, $r$ is the radius of the cup, and $h$ is the height of the cup.
See see \figref{model} for an illustration.

Rather than performing a computationally expensive dynamic simulation with a spring-mass model to determine when seal formation is feasible,
we make simplifying assumptions to evaluate seal formation geometrically.
Specifically, we compute a configuration of $\mC$ that achieves a seal by projecting $\mC$ onto the surface of the target object's triangular mesh $\mM$ and evaluate the feasibility of that configuration under quasi-static conditions as a proxy for the dynamic feasibility of seal formation.

In our model, $\mC$ has two types of springs -- \textit{structural} springs that represent the physical structure of the suction cup and \textit{flexion} springs that do not correspond to physical structures but instead are used to resist bending along the cup's surface.
Dynamic spring-mass systems with similar structures have been used in prior work to model stiff sheets of rubber~\cite{provot1995deformation}.
The undeformed structural springs of $\mC$ form a right pyramid with height $h$ and with a base that is a regular $n$-gon with circumradius $r$.
Let $\mathcal{V} = \{\mathbf{v}_1, \mathbf{v}_2, \ldots, \mathbf{v}_n, \mathbf{a}\}$ be the set of vertices of the undeformed right pyramid, where each $\mathbf{v}_i$ is a base vertex and $\mathbf{a}$ is the pyramid's apex.
We define the model's set of springs as follows:
\begin{itemize}
    \item \textbf{Perimeter (Structural) Springs} -- Springs linking vertex $\mathbf{v}_i$ to vertex $\mathbf{v}_{i+1}$, $\forall i \in \{1, \ldots, n\}$.
    \item \textbf{Cone (Structural) Springs} -- Springs linking vertex $\mathbf{v}_i$ to vertex $\mathbf{a}$, $\forall i \in \{1, \ldots, n\}$.
    \item \textbf{Flexion Springs} -- Springs linking vertex $\mathbf{v}_i$ to vertex $\mathbf{v}_{i+2}$, $\forall i \in \{1, \ldots, n\}$.
\end{itemize}

In the model, a complete seal is formed between $\mC$ and $\mM$ if and only if each of the perimeter springs of $\mC$ lies entirely on the surface of $\mM$.
Given a target mesh $\mM$ with a target grasp $\bu = (\bp, \bv)$ for the gripper,
we choose an initial configuration of $\mC$ such that $\mC$ is undeformed and the approach line $(\mathbf{p}, \mathbf{v})$ passes through $\mathbf{a}$ and is orthogonal to the base of $\mC$.
Then, we make the following assumptions to determine a final static contact configuration of $\mC$ that forms a complete seal against $\mM$ (see \figref{spring-approach}):
\begin{itemize}
    \item \textbf{The perimeter springs} of $\mC$ must not deviate from the original undeformed regular $n$-gon when projected onto a plane orthogonal to $\mathbf{v}$.
        This means that their locations can be computed by projecting them along $\mathbf{v}$ from their original locations onto the surface of $\mM$.
    \item \textbf{The apex}, $\mathbf{a}$, of $\mC$ must lie on the approach line $(\mathbf{p}, \mathbf{v})$ and, given the locations of $\mC$'s base vertices, must also lie at a location that keeps the average distance along $\mathbf{v}$ between $\mathbf{a}$ and the perimeter vertices equal to $h$.
\end{itemize}
See the supplemental file for additional details.

Given this configuration, a seal is feasible if:
\begin{itemize}
    \item The cone faces of $\mC$ do not collide with $\mM$ during approach or in the contact configuration.
    \item The surface of $\mM$ has no holes within the contact ring traced out by $\mC$'s perimeter springs.
    \item The energy required in each spring to maintain the contact configuration of $\mC$ is below a real-valued threshold $E$ modeling the maximum deformation of the suction cup material against the object surface.
\end{itemize}
%For the final requirement, we use a per-spring threshold of a $10\%$ change in length, which was used as the spring stretch limit in ~\cite{provot1995deformation}.
We threshold the energy in individual springs rather than the total energy for $\mC$ because air gaps are usually caused by local geometry.
% -- for example, a small groove in the mesh's surface would require a significant amount of deformation energy from small sections of the cup, but the overall deformation energy could still be low if the rest of the mesh is planar.

\subsection{Wrench Space Analysis}
To determine the degree to which the suction cup can resist external wrenches such as gravity, we analyze the set of wrenches that the suction cup can apply.

\subsubsection{Wrench Resistance}
The object wrench set for a grasp using a contact model with $m$ basis wrenches is $\Lambda = \{\bw \in \bR^6 \mid \bw =  G \alpha \text{ for some } \alpha \in \mF\}$, where $G \in \bR^{6 \times m}$ is a set of basis wrenches in the object coordinate frame, and $\mF \subseteq \bR^m$ is a set of constraints on contact wrench magnitudes~\cite{ murray1994mathematical}.
\begin{definition}
A grasp $\bu$ achieves {\it wrench resistance} with respect to $\bw$ if $-\bw \in \Lambda$~\cite{krug2017grasp, murray1994mathematical}.
\end{definition}
\noindent We encode wrench resistance as a binary variable $W$ such that $W=0$ if $\bu$ resists $\bw$ and $W=0$ otherwise.

\subsubsection{Suction Contact Model}
Many suction contact models acknowledge normal forces, vacuum forces, tangential friction, and torsional friction~\cite{bahr1996design, kolluru1998modeling, mantriota2007theoretical, stuart2015suction} similar to a point contact with friction or soft finger model~\cite{murray1994mathematical}.
However, under this model, a single suction cup cannot resist torques about axes in the contact tangent plane, implying that any torque about such axes would cause the suction cup to drop an object (see the supplementary material for a detailed proof).
This defies our intuition since empirical evidence suggests that a single point of suction can robustly transport objects~\cite{eppner2016lessons, hernandez2016team}.

We hypothesize that these torques are resisted through an asymmetric pressure distribution on the ring of contact between the suction cup and object, which occurs due to passive elastic restoring forces in the material.
\figref{model} illustrates the suction ring contact model.
The grasp map $G$ is defined by the following basis wrenches:
\begin{enumerate}
	\item {\bf Actuated Normal Force ($f_z$):} The force that the suction cup material applies by pressing into the object along the contact $z$ axis.
	\item {\bf Vacuum Force ($V$):} The magnitude of the constant force pulling the object into the suction cup coming from the air pressure differential.
	\item {\bf Frictional Force ($f_f = (f_x, f_y)$):} The force in the contact tangent plane due to the normal force between the suction cup and object, $f_N = f_z + V$. 
	\item {\bf Torsional Friction ($\tau_z$):} The torque resulting from frictional forces in the ring of contact.
	\item {\bf Elastic Restoring Torque ($\tau_e = (\tau_x, \tau_y)$):} The torque about axes in the contact tangent plane resulting from elastic restoring forces in the suction cup pushing on the object along the boundary of the contact ring.
\end{enumerate}

The magnitudes of the contact wrenches are constrained due to (a) the friction limit surface~\cite{kao2008contact}, (b) limits on the elastic behavior of the suction cup material, and (c) limits on the vacuum force.
In the suction ring contact model, $\mF$ is approximated by a set of linear constraints for efficient computation of wrench resistance:
\begin{align*}
	\text{{\bf Friction:}} && \sqrt{3} | f_x | &\leq \mu f_N & \sqrt{3} | f_y | &\leq \mu f_N & \sqrt{3} | \tau_z| &\leq r \mu f_N  \\
	\text{{\bf Material:}} && \sqrt{2}| \tau_x | &\leq \pi  r \kappa & \sqrt{2} | \tau_y | &\leq \pi  r \kappa \\
	\text{{\bf Suction:}} &&  f_z &\geq -V
\end{align*}
\noindent Here $\mu$ is the friction coefficient, $r$ is the radius of the contact ring, and $\kappa$ is a material-dependent constant modeling the maximum stress for which the suction cup has linear-elastic behavior.
These constraints define a subset of the friction limit ellipsoid and cone of admissible elastic torques under a linear pressure distribution about the ring of the cup.
Furthermore, we can compute wrench resistance using quadratic programming due to the linearity of the constraints.
See the supplemental file for a detailed derivation and proof.

\subsection{Robust Wrench Resistance}
\seclabel{robust-wr}
We evaluate the robustness of candidate suction grasps by evaluating seal formation and wrench resistance over distributions on object pose, grasp pose, and disturbing wrenches:
\begin{definition}
The {\it robust wrench resistance} metric for $\bu$ and $\bx$ is $\lambda(\bu, \bx) = \mathbb{P}(W \mid \bu, \bx)$, the probability of success under perturbations in object pose, gripper pose, friction, and disturbing wrenches.
\end{definition}
In practice, we evaluate robust wrench resistance by taking $J$ samples, evaluating binary wrench resistance for each, and computing the sample mean: $\frac{1}{J} \sum_{j=1}^J W_j$.

\section{Dex-Net 3.0 Dataset}
\seclabel{dataset}

\begin{figure}[t!]
\centering
\includegraphics[scale=0.22]{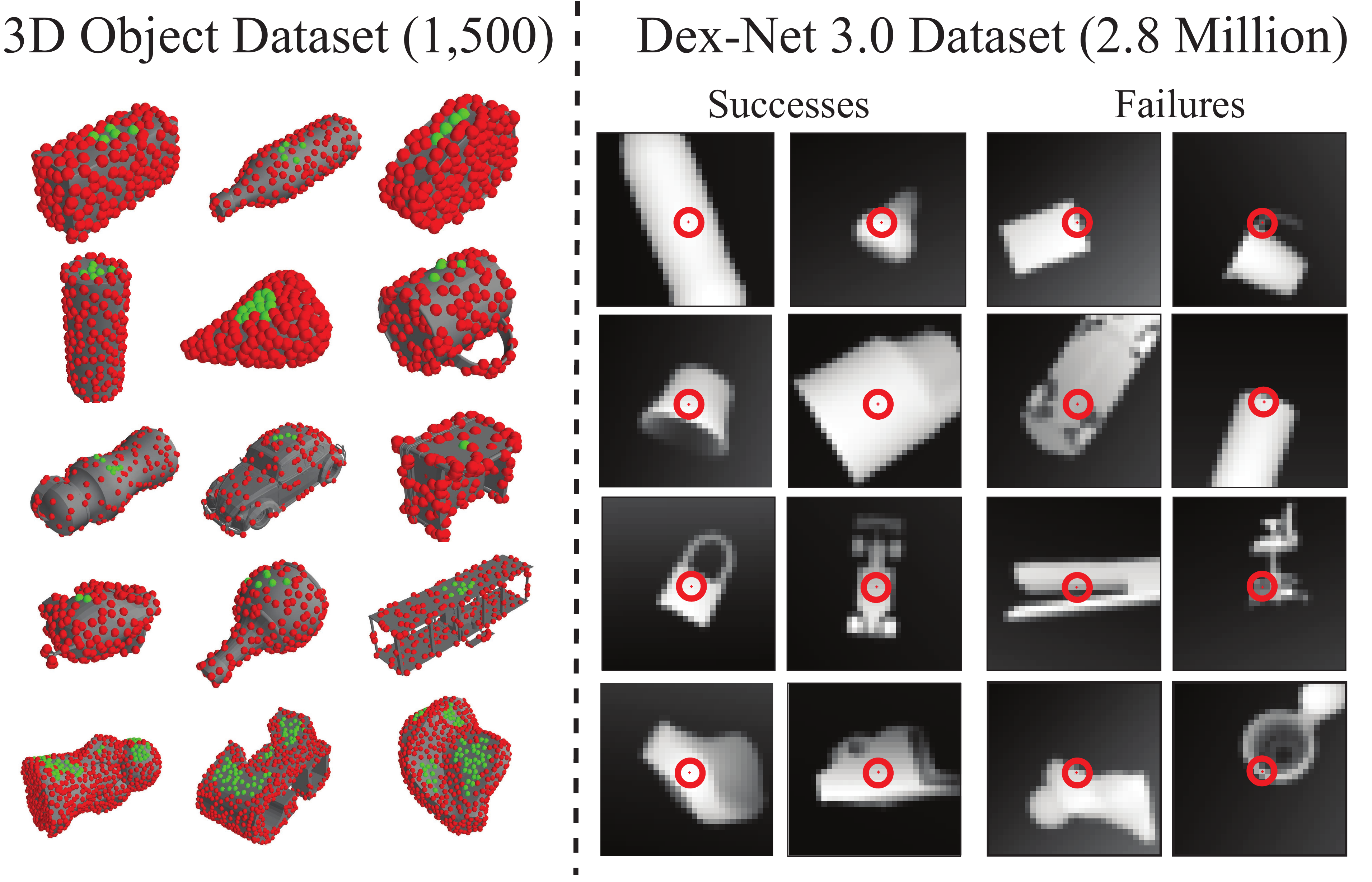}
\caption{The Dex-Net 3.0 dataset. (Left) The Dex-Net 3.0 object dataset contains approximately 350k unique suction target points across the surfaces of 1,500 3D models from the KIT object database~\cite{kasper2012kit} and 3DNet~\cite{wohlkinger20123dnet}. Each suction grasp is classified as robust (green) or non-robust (red). Robust grasps are often above the object center-of-mass on flat surfaces of the object. (Right) The Dex-Net 3.0 point cloud dataset contains 2.8 million tuples of point clouds and suction grasps with robustness labels, with approximately 11.8$\%$ positive examples.}
\figlabel{dataset}
\end{figure}

\begin{figure*}[t!]
\centering
\includegraphics[scale=0.22]{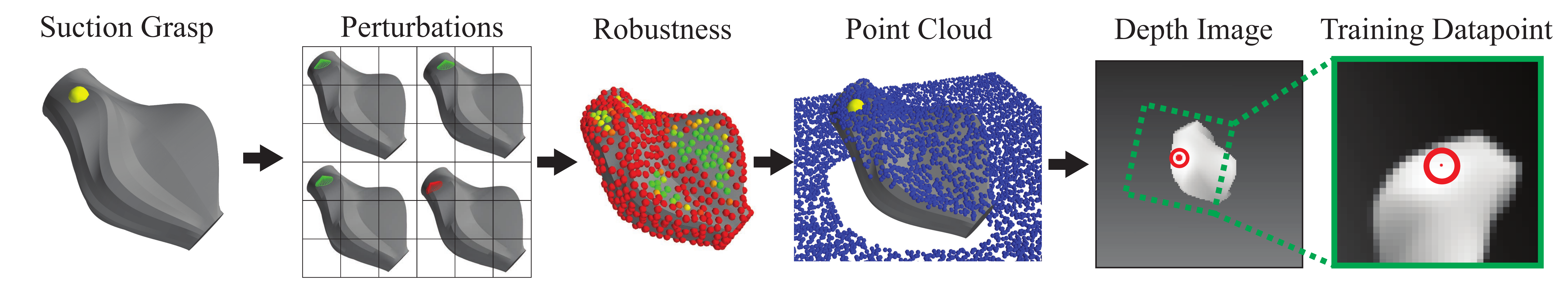}
\caption{Pipeline for generating the Dex-Net 3.0 dataset (left to right).
We first sample a candidate suction grasp from the object surface and evaluate the ability to form a seal and resist gravity over perturbations in object pose, gripper pose, and friction.
The samples are used to estimate the probability of success, or robustness, for candidate grasps on the object surface.
We render a point cloud for each object and associate the candidate grasp with a pixel and orientation in the depth image through perspective projection.
Training datapoints are centered on the suction target pixel and rotated to align with the approach axis to encode the invariance of the robustness to image locations.}
\figlabel{pipeline}
\end{figure*}

To learn to predict grasp robustness based on noisy point clouds, we generate the Dex-Net 3.0 training dataset of point clouds, grasps, and grasp reward labels by sampling tuples $(R_i, \bu_i, \by_i)$ from a joint distribution $p(R, \bx, \by, \bu)$ modeled as the product of distributions on:
\begin{itemize}
	\item {\bf States: $p(\bx)$:} A uniform distribution over a discrete dataset of objects and their stable poses and uniform continuous distributions over the object planar pose and camera poses in a bounded region of the workspace.
	\item {\bf Grasp Candidates: $p(\bu | \bx)$:} A uniform random distribution over contact points on the object surface.
	\item {\bf Grasp Rewards $p(R \mid \bu, \bx)$:} A stochastic model of wrench resistance for the gravity wrench that is sampled by perturbing the gripper pose according to a Gaussian distribution and evaluating the contact model described in \secref{contact-model}.
	\item {\bf Observations $p(\by \mid \bx)$:} A depth sensor noise model with multiplicative and Gaussian process pixel noise.
\end{itemize}
\noindent \figref{dataset} illustrates a subset of the Dex-Net 3.0 object and grasp dataset.
The parameters of the sampling distributions and compliant suction contact model $(n, r, h, E, V, \mu, \kappa, \epsilon)$ (see \secref{contact-model}) were set by maximizing average precision of the $Q$ values using grid search for a set of grasps attempted on an ABB YuMi robot on a set of known 3D printed objects (see \secref{object-datasets}).

Our pipeline for generating training tuples is illustrated in \figref{pipeline}.
We first sample state by selecting an object at random from a database of 3D CAD models and sampling a friction coefficient, planar object pose, and camera pose relative to the worksurface.
We generate a set of grasp candidates for the object by sampling points and normals uniformly at random from the surface of the object mesh.
We then set the binary reward label $R=1$ if a seal is formed and robust wrench resistance (described in \secref{robust-wr}) is above a threshold value $\epsilon$.
Finally, we sample a point cloud of the scene using rendering and a model of image noise~\cite{mahler2016dexnet}.
The grasp success labels are associated with pixel locations in images through perspective projection~\cite{hartley2003multiple}.
A graphical model for the sampling process and additional details on the distributions can be found in the supplemental file.

\section{Learning a Deep Robust Grasping Policy}
\seclabel{gqcnn}

We use the Dex-Net 3.0 dataset to train a GQ-CNN that takes as input a single-view point cloud of an object resting on the table and a candidate suction grasp defined by a target 3D point and approach direction, and outputs the robustness, or estimated probability of success, for the grasp on the visible object.

Our GQ-CNN architecture is identical to Dex-Net 2.0~\cite{mahler2017dex} except that we modify the pose input stream to include the angle between the approach direction and the table normal.
The point cloud stream takes a depth image centered on the target point and rotated to align the middle column of pixels with the approach orientation similar to a spatial transforming layer~\cite{jaderberg2015spatial}.
%The transformed point cloud is convolved with four layers of learned convolutional filters motivated by the correlation of surface planarity with suction success~\cite{domae2014fast, hernandez2016team}.
The end-effector depth from the camera and orientation are input to a fully connected layer in a separate pose stream and concatenated with conv features in a fully connected layer.
We train the GQ-CNN with using stochastic gradient descent with momentum using an 80-20 training-to-validation image-wise split of the Dex-Net 3.0 dataset.
Training took approximately 12 hours on three NVIDIA Titan X GPUs.
%We use the weight initialization, data augmentation, and point cloud noise sampling strategies from the parallel-jaw GQ-CNN of Dex-Net 2.0~\cite{mahler2017dex}. 
The learned GQ-CNN achieves 93.5$\%$ classification accuracy on the held-out validation set.

We use the GQ-CNN in a deep robust grasping policy to plan suction target grasps from point clouds on a physical robot.
The policy uses the Cross Entropy Method (CEM)~\cite{levine2016learning, mahler2017dex, rubinstein2013fast}.
CEM samples a set of initial candidate grasps uniformly at random from the set of surface points and inward-facing normals on a point cloud of the object, then iteratively resamples grasps from a Gaussian Mixture Model fit to the grasps with the highest predicted probability of success.
See the supplemental file for example grasps planned by the policy.
%\figref{policy-examples} illustrates the probability of success predicted by the GQ-CNN on candidates grasps from several real point clouds.

%\begin{figure}[t!]
%\centering
%\includegraphics[scale=0.30]{figures/illustrations/robustness_maps-01.eps}
%\caption{Robust grasps planned with the Dex-Net 3.0 GQ-CNN-based policy on example RGB-D point clouds. \textbf{(Left)} The robot is presented an object in isolation. \textbf{(Middle)} Initial candidate suction target points colored by the predicted probability of success from zero (red) to one (green). Robust grasps tend to concentrate around the object centroid. \textbf{(Right)} The policy optimizes for the grasp with the highest probability of success using the Cross Entropy Method.
%\vspace*{-20pt}}
%\figlabel{policy-examples}
%\end{figure}

\section{Experiments}
\seclabel{experiments}

\begin{figure*}[t!]
\centering
\includegraphics[scale=0.3]{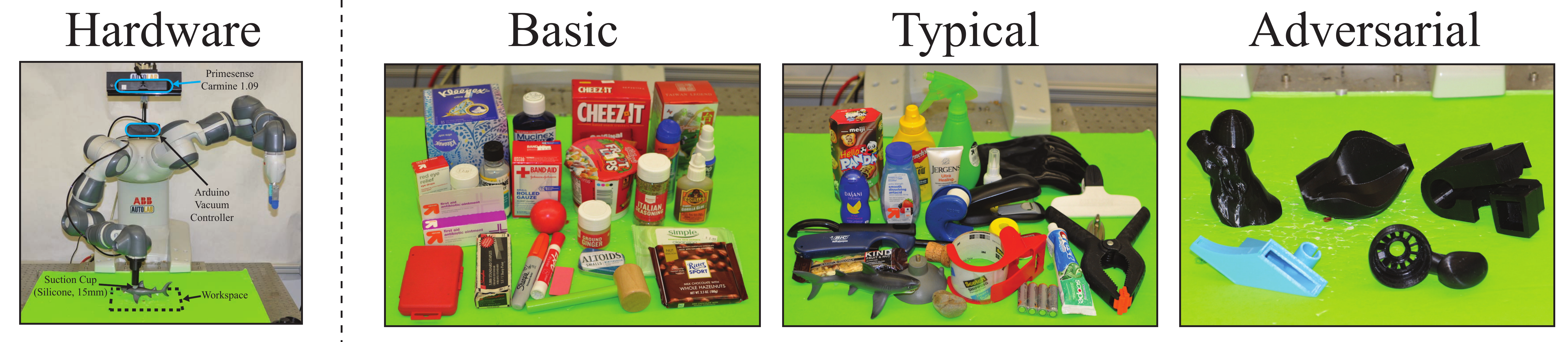}
\caption{\textbf{(Left)} The experimental setup with an ABB YuMi equipped with a suction gripper. \textbf{(Right)} The 55 objects used to evaluate suction grasping performance. The objects are divided into three categories to characterize performance: Basic (e.g. prismatic objects), Typical, and Adversarial.}
\figlabel{object-datasets}
\end{figure*}

We ran experiments to characterize the precision of robust wrench resistance when object shape and pose are known and the precision of our deep robust grasping policy for planning grasps from point clouds for three object classes.
%\begin{enumerate}
%    \item The correlation between suction robustness predicted by our model and real-world successes when object shape and pose are known.
%    \item The performance of a GQ-CNN-based policy for planning robust grasps from point clouds over a variety of objects with unknown shapes and poses.
%\end{enumerate}

\subsection{Object Classes}
\seclabel{object-datasets}
We created a dataset of 55 rigid and non-porous objects including tools, groceries, office supplies, toys, and 3D printed industrial parts.
We separated objects into three categories, illustrated in \figref{object-datasets}:
\begin{enumerate}
    \item {\it Basic:} Prismatic solids (e.g. rectangular prisms, cylinders). Includes 25 objects.
    \item {\it Typical:} Common objects with varied geometry and many accessible, approximately planar surfaces. Includes 25 objects.
    \item {\it Adversarial:} 3D-printed objects with complex geometry (e.g. curved or narrow surfaces) that are difficult to access. Includes 5 objects.
\end{enumerate}
%The 3D-printed adversarial objects allow us to control for object shape, pose, and mass.  
For object details, see http://bit.ly/2xMcx3x.

\subsection{Experimental Protocol}
We ran experiments with an ABB YuMi with a Primesense Carmine 1.09 and a suction system with a 15$mm$ diameter silicone single-bellow suction cup and a VM5-NC VacMotion vacuum generator with a payload of approximately 0.9$kg$.
The experimental workspace is illustrated in the left panel of \figref{object-datasets}.
In each experiment, the operator iteratively presented a target object to the robot and the robot planned and executed a suction grasp on the object.
The operator labeled successes based on whether or not the robot was able to lift and transport the object to the side of the workspace.
For each method, we measured:
\begin{enumerate}
	\item {\bf Average Precision (AP).} The area under the precision-recall curve, which measures precision over possible thresholds on the probability of success predicted by the policy. This is useful for industrial applications where a robot may take an alternative action (e.g. asking for help) if the planned grasp is predicted to fail.
	\item {\bf Success Rate.} The fraction of all grasps that were successful.
\end{enumerate}
All experiments ran on a Desktop running Ubuntu 14.04 with a 2.7 GHz Intel Core i5-6400 Quad-Core CPU and an NVIDIA GeForce 980 GPU.

%\subsection{Average Precision Metric}
%The success rate measures the fraction of grasps that successfully lift and transport the target object and is commonly used to assess performance.
%In practice, however, grasping policies would be part of a larger system that can decide whether to execute a planned grasp or take a different action (e.g. asking a human for help, poking objects to make other grasps available) by thresholding the predicted probability of success.
%Given a confidence threshold for grasp success, we use two performance metrics:
%\begin{enumerate}
%	\item {\bf Precision.} The fraction of grasps predicted to succeed that succeed when executed.
%	\item {\bf High Confidence Rate.} The fraction of all grasps that are predicted to succeed.
%\end{enumerate}
%There is typically a tradeoff between these metrics: a higher confidence threshold will reduce false positives but will also reduce the frequency of grasp attempts, increasing runtime and decreasing the diversity of cases that the robot is able to successfully handle.
%To assess this tradeoff we measure the Average Precision (AP), or the mean precision value over all possible confidence thresholds from 0 to 1.

\subsection{Performance on Known Objects}
\seclabel{known}
To assess performance of our robustness metric independently from the perception system, we evaluated whether or not the metric was predictive of suction grasp success when object shape and pose were known using the 3D printed Adversarial objects (right panel of \figref{object-datasets}).
The robot was presented one of the five Adversarial objects in a known stable pose, selected from the top three most probable stable poses.
We hand-aligned the object to a template image generated by rendering the object in a known pose on the table.
Then, we indexed a database of grasps precomputed on 3D models of the objects and executed the grasp with the highest metric value for five trials.
In total, there were 75 trials per experiment.

We compared the following metrics:
\begin{enumerate}
\item {\bf Planarity-Centroid (PC3D).} The inverse distance to the object centroid for sufficiently planar patches on the 3D object surface.
\item {\bf Spring Stretch (SS).} The maximum stretch among virtual springs in the suction contact model.
\item {\bf Wrench Resistance (WR).} Our model without perturbations.
\item {\bf Robust Wrench Resistance (RWR).} Our model.
\end{enumerate}
\noindent The RWR metric performed best with 99$\%$ AP compared to 93$\%$ AP for WR, 89$\%$ AP for SS, and 88$\%$ for PC3D.
%The results are detailed in \tabref{correlation}.
%A policy based on robust wrench resistance metric achieved nearly 100$\%$ average precision and 80$\%$ success on this dataset, suggesting that the ranking of grasps by robust wrench resistance is correlated with the ranking by physical successes.

%\begin{table}[t]
%\centering
%\resizebox{\textwidth}{!}{%
%        \begin{tabular}{r || R{1.0cm} | R{1.0cm}}
%       \multicolumn{1}{c ||}{\specialcell{{\bf Metric}}}&  \multicolumn{1}{c |}{\specialcell{\bf AP ($\%$)}} & \multicolumn{1}{c}{\specialcell{\bf Success Rate ($\%$)}} \\
%       \hline & & \\
%       PC & 88 &  80  \\
%       SS & 89 &  84 \\
%       WR & 93 & 80  \\
%       RWR &{\bf 100} & {\bf 92}
%        \end{tabular}
%        }
%		\caption{Performance of robust grasping policies with known state (3D object shape and pose) across 75 physical trials per policy on the Adversarial object dataset. The policies differ by the metric used to rank grasps, and each metric is computed using the known 3D object geometry. The robust wrench resistance metric, which considers the ability of a suction cup to form a seal and resist gravity under perturbations, has very high precision. In comparison, the Planarity-Centroid heuristic achieves only 88$\%$ precision and 80$\%$ success.}
%		\tablabel{correlation}
%\end{table}

\subsection{Performance on Novel Objects}
\seclabel{novel}
We also evaluated the performance of GQ-CNNs trained on Dex-Net 3.0 for planning suction target points from a single-view point cloud.
In each experiment, the robot was presented one object from either the Basic, Typical, or Adversarial classes in a pose randomized by shaking the object in a box and placing it on the table.
The object was imaged with a depth sensor and segmented using 3D bounds on the workspace.
Then, the grasping policy executed the most robust grasp according to a success metric.
In this experiment the human operators were blinded from the method they were evaluating to remove bias in human labels.

We compared policies that optimized the following metrics:
\begin{enumerate}
\item {\bf Planarity.} The inverse sum of squared errors from an approach plane for points within a disc with radius equal to that of the suction cup.
\item {\bf Centroid.} The inverse distance to the object centroid.
\item {\bf Planarity-Centroid (PC).} The inverse distance to the centroid for planar patches on the 3D object surface.
\item {\bf GQ-CNN (ADV).} Our GQ-CNN trained on synthetic data from the Adversarial objects (to assess the ability of the model to fit complex objects).
\item {\bf GQ-CNN (DN3).} Our GQ-CNN trained on synthetic data from 3DNet~\cite{wohlkinger20123dnet}, KIT~\cite{kasper2012kit}, and the Adversarial objects.
\end{enumerate} 

\tabref{policy-results} details performance on the Basic, Typical, and Adversarial objects.
On the Basic and Typical objects, we see that the Dex-Net 3.0 policy is comparable to PC in terms of success rate and has near-perfect AP, suggesting that failed grasps often have low robustness and can therefore be detected.
On the adversarial objects, GQ-CNN (ADV) significantly outperforms GQ-CNN (DN3) and PC, suggesting that this method can be used to successfully grasp objects with complex surface geometry as long as the training dataset closely matches the objects seen at runtime.
The DN3 policy took an average of 3.0 seconds per grasp.

\begin{table*}[t]
\centering
\resizebox{\textwidth}{!}{%
        \begin{tabular}{r || R{1.0cm} | R{1.0cm} || R{1.0cm} | R{1.0cm} || R{1.0cm} | R{1.0cm}}
       \multicolumn{1}{c ||}{\specialcell{}}& \multicolumn{2}{c ||}{\specialcell{\bf Basic}}& \multicolumn{2}{c ||}{\specialcell{\bf Typical}}& \multicolumn{2}{c}{\specialcell{\bf Adversarial}}  \\
       \hline & & & & & & \\
       \multicolumn{1}{c ||}{\specialcell{}}&  \multicolumn{1}{c |}{\specialcell{\bf AP ($\%$)}} & \multicolumn{1}{c ||}{\specialcell{\bf Success Rate ($\%$)}}&  \multicolumn{1}{c |}{\specialcell{\bf AP ($\%$)}} & \multicolumn{1}{c ||}{\specialcell{\bf Success Rate ($\%$)}}&  \multicolumn{1}{c |}{\specialcell{\bf AP ($\%$)}} & \multicolumn{1}{c}{\specialcell{\bf Success Rate ($\%$)}} \\
       \hline & & & & & & \\
       {\bf Planarity} & 81 & 74 & 69 & 67 & 48 & 47 \\
       {\bf Centroid} & 89 & 92 & 80 & 78 & 47 & 38 \\
       {\bf Planarity-Centroid} & 98 &  94 & 94 & {\bf 86} & 64 & 62 \\
       {\bf GQ-CNN (ADV)} & 83 & 77 & 75 & 67 & {\bf 86} & {\bf 81} \\
       {\bf GQ-CNN (DN3)} &{\bf 99} & {\bf 98} & {\bf 97} & 82 & 61 & 58 
        \end{tabular}
        }
        \caption{Performance of point-cloud-based grasping policies for 125 trials each on the Basic and Typical datasets and 100 trials each on the Adversarial dataset. We see that the GQ-CNN trained on Dex-Net 3.0 has the highest Average Precision (AP) (area under the precision-recall curve) on the Basic and Typical objects, suggesting that the robustness score from the GQ-CNN could be used to anticipate grasp failures and select alternative actions (e.g. probing objects) in the context of a larger system. Also, a GQ-CNN trained on the Adversarial dataset outperforms all methods on the Adversarial objects, suggesting that the performance of our model is improved when the true object models are used for training. }
		\tablabel{policy-results}
%\vspace*{-20pt}
\end{table*}

%\begin{figure*}[t!]
%\centering
%\includegraphics[scale=0.6]{figures/illustrations/hypothesis2-01.eps}
%\caption{Precision versus percent of grasps attempted for all confidence thresholds on the predicted probability of grasp success for 125 trials on the Basic, Typical, and Adversarial object datasets. The GQ-CNN trained on Dex-Net 3.0 has near 100$\%$ precision on the Basic and Typical datasets for a signficant portion of attempts, suggesting that the GQ-CNN is able to predict when it is likely to fail on novel objects. The GQ-CNN trained on the Adversarial objects has a significantly higher precision on the Adversarial dataset, butdoes not perform as well on the other objects. \TODO{Fix DEXNET in caption}
%\vspace*{-20pt}}
%\figlabel{prec-conf-gen}
%\end{figure*}

\subsection{Failure Modes}
The most common failure mode was attempting to form a seal on surfaces with surface geometry that prevent seal formation.
This is partially due to the limited resolution of the depth sensor, as our seal formation model is able to detect the inability to form a seal on such surfaces when the geometry is known precisely.
In contrast, the planarity-centroid metric performs poorly on objects with non-planar surfaces near the object centroid.

\section{Future Work}
\seclabel{discussion}
In future work we will study sensitivity to (1) the distribution of 3D object models using in the training dataset, (2) noise and resolution in the depth sensor, and (3) variations in vacuum suction hardware (e.g. cup shape, hardness of cup material).
We will also extend this model to learning suction grasping policies for bin-picking with heaps of parts and to composite policies that combine suction grasping with parallel-jaw grasping by a two-armed robot.
We are also working with colleagues in the robot grasping community to propose shareable benchmarks and protocols that specify experimental objects and conditions with industry-relevant metrics such as Mean Picks Per Hour (MPPH), see http://goo.gl/6M5rfw.

%In future work we plan to study the sensitivity of the method to (1) the distribution of 3D object models using in the training dataset, (2) depth sensor noise and resolution, and (3) different suction cup hardware such as non-circular cups, cups with larger diameters, and cups with harder or softer material.
%To further quantify performance we are also working to establish benchmarks and protocols that can be used to compare the mean picks per hour (MPPH) of grasp planning methods across different lab hardware setups.
%Finally, we are also interested in using the suction grasping policy as one of multiple primitives for bin picking, where a high-level policy decides between using parallel jaws or suction based on the predicted probability of success.

\appendices
\section{Additional Experiments}
\seclabel{additional-experiments}
To better characterize the correlation of our robust wrench resistance metric, compliant suction contact model, and GQ-CNN-based policy for planning suction target grasps from point clouds with physical outcomes on a real robot, we present several additional analyses and experiments.
\subsection{Performance Metrics}
Our primary numeric metrics of performance were:
\begin{enumerate}
	\item {\bf Average Precision (AP).} The area under the precision-recall curve, which measures precision over possible thresholds on the probability of success predicted by the policy. This is useful for industrial applications where a robot may take an alternative action (e.g. probing, asking for help) if the planned grasp is predicted to fail.
	\item {\bf Success Rate.} The fraction of all grasps that were successful.
\end{enumerate}

We argue that these metrics alone do not give a complete picture of how useful a suction grasp policy would work in practice.
Average Precision (AP) penalizes a policy for having poor recall (a high rate of false negatives relative to true positives), and success rate penalizes a policy with a high number of failures.
However, not all failures should be treated equally: some failures are predicted to occur by the GQ-CNN (low predicted probability of success) while the others are the result of an overconfident prediction.

In practice, a suction grasp policy would be part of a larger system (e.g. a state machine) that could decide whether or not to execute a grasp based on the continuous probability of success output by the GQ-CNN.
As long as the policy is not overconfident, such as system can detect failures before they occur and take an alternative action such as attempting to turn the object over, asking a human for help, or leaving the object in the bin for error handling.
At the same time, if a policy is too conservative and never predicts successes, then the system will handle be able to handle very few test cases.

We illustrate this tradeoff by plotting the Success-Attempt Rate curve which plots:
\begin{enumerate}
	\item {\bf Success Rate.} The fraction of fraction of grasps that are successful if the system only executes grasps have predicted probability of success greater than a confidence threshold $\tau$.
	\item {\bf Attempt Rate.} The fraction of all test cases for which the system attempts a grasp, if the system only attempts grasps with predicted probability of success greater than a confidence threshold $\tau$.
\end{enumerate}
\noindent over all possible values of the confidence threshold $\tau$.
There is typically an inverse relationship between the two metrics: a higher confidence threshold will reduce false positives but will also reduce the frequency of grasp attempts, increasing runtime and decreasing the diversity of cases that the robot is able to successfully handle.

\subsection{Performance on Known Objects}
\seclabel{additional-known-objects}
To assess performance of our robustness metric independent of the perception system, we evaluated whether or not the metric was predictive of suction grasp success when object shape and pose were known using the 3D printed Adversarial objects.
The robot was presented one of the five Adversarial objects in a known stable pose, selected from the top three most probable stable poses. We hand-aligned the object to a template image generated by rendering the object in a known pose on the table.
Then, we indexed a database of grasps precomputed on 3D models of the objects and executed the grasp with the highest metric value for five trials.
In total, there were 75 trials per experiment.

We compared the following metrics:
\begin{enumerate}
\item {\bf Planarity-Centroid (PC3D).} The inverse distance to the object centroid for sufficiently planar patches on the 3D object surface.
\item {\bf Spring Stretch (SS).} The maximum stretch among virtual springs in the suction contact model.
\item {\bf Wrench Resistance (WR).}
\item {\bf Robust Wrench Resistance (RWR).}
\end{enumerate}

The results are detailed in \tabref{correlation} and the Success vs Attempt Rate curve is plotted in \figref{hyp1}.
A policy based on the robust wrench resistance metric achieved nearly 100$\%$ average precision and 92$\%$ success on this dataset, suggesting that the ranking of grasps by robust wrench resistance is correlated with the ranking by physical successes.

\begin{table}[t]
\centering
\resizebox{\textwidth}{!}{%
        \begin{tabular}{r || R{1.0cm} | R{1.0cm}}
       \multicolumn{1}{c ||}{\specialcell{{\bf Metric}}}&  \multicolumn{1}{c |}{\specialcell{\bf AP ($\%$)}} & \multicolumn{1}{c}{\specialcell{\bf Success Rate ($\%$)}} \\
       \hline & & \\
       PC3D & 88 &  80  \\
       SS & 89 &  84 \\
       WR & 93 & 80  \\
       RWR &{\bf 100} & {\bf 92}
        \end{tabular}
        }
		\caption{Performance of robust grasping policies with known state (3D object shape and pose) across 75 physical trials per policy on the Adversarial object dataset. The policies differ by the metric used to rank grasps, and each metric is computed using the known 3D object geometry. The robust wrench resistance metric, which considers the ability of a suction cup to form a seal and resist gravity under perturbations, has very high precision. In comparison, the Planarity-Centroid heuristic achieves only 88$\%$ precision and 80$\%$ success.}
		\tablabel{correlation}
\end{table}

\begin{figure}[t!]
\centering
\includegraphics[scale=0.7]{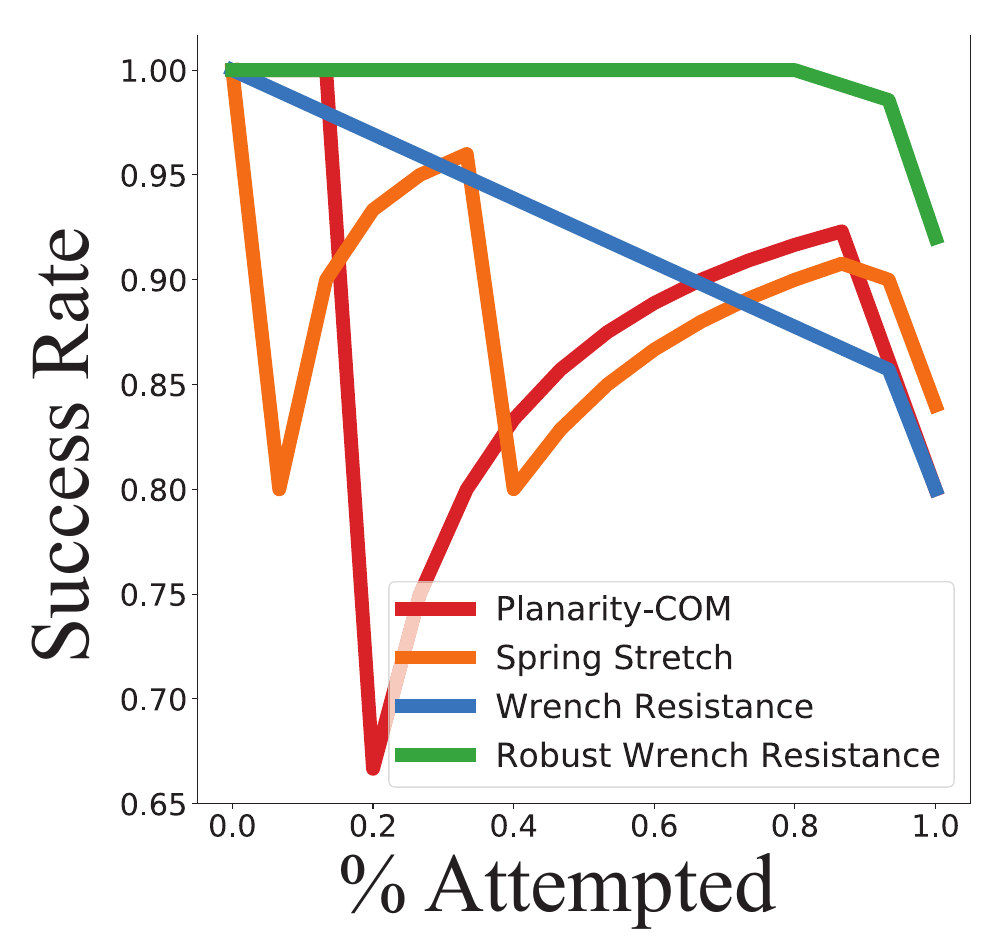}
\caption{Success rate vs Attempt Rate for grasp quality metrics on known 3D objects in known poses. The data was collected across 75 trials per policy on the Adversarial object dataset. The robust wrench resistance metric based on our compliant suction contact model had a 100$\%$ success rate for a large percentage of possible test cases, whereas a heuristic based on planarity and the distance to the object center of mass had success rates as low as $67\%$, indicating that the real-valued distance to the center of mass is not well correlated with grasp success.}
\figlabel{hyp1}
\end{figure}

\subsection{Performance on Novel Objects}

We also evaluated the performance of GQ-CNNs trained on Dex-Net 3.0 for planning suction target points from a single-view point cloud.
In each experiment, the robot was presented one object from either the Basic, Typical, or Adversarial classes in a pose randomized by shaking the object in a box and placing it on the table.
The object was imaged with a depth sensor and segmented using 3D bounds on the workspace.
Grasp candidates were then sampled from the depth image and the grasping policy executed the most robust candidate grasp according to a success metric.
In this experiment the human operators were blinded from the method they were evaluating to remove bias in human labels.

We compared policies that used the following metrics:
\begin{enumerate}
\item {\bf Planarity.} The inverse sum of squared errors from an approach plane for points within a disc with radius equal to that of the suction cup.
\item {\bf Centroid.} The inverse distance to the object centroid.
\item {\bf Planarity-Centroid.} The inverse distance to the centroid for sufficiently planar patches on the 3D object surface.
\item {\bf GQ-CNN (ADV).} A GQ-CNN trained on synthetic data from only the Adversarial objects (to assess the ability of the model to fit complex objects).
\item {\bf GQ-CNN (DN3).} A GQ-CNN trained on synthetic data from the 3DNet~\cite{wohlkinger20123dnet}, KIT~\cite{kasper2012kit}, and Adversarial object datasets.
\end{enumerate} 

\tabref{policy-results} details performance on the Basic, Typical, and Adversarial objects, and \figref{hyp2} illustrates the Success-Attempt Rate tradeoff.
We see that the Dex-Net 3.0 policy has the highest AP across the Basic and Typical classes.
Also, the GQ-CNN trained on the Adversarial objects significantly outperforms all methods on the Adversarial dataset, suggesting that our model is able to exploit knowledge of complex 3D geometry to plan robust grasps.
Furthermore, the Success-Attempt Rate curve suggests that the continuous probability of success output by the Dex-Net 3.0 policy is highly correlated with the true success label and can be used to detect failures before they occur on the Basic and Typical object classes.
The Dex-Net 3.0 policy took an average of approximately 3 seconds to plan each grasp.

\begin{table*}[t]
\centering
\resizebox{\textwidth}{!}{%
        \begin{tabular}{r || R{1.0cm} | R{1.0cm} || R{1.0cm} | R{1.0cm} || R{1.0cm} | R{1.0cm}}
       \multicolumn{1}{c ||}{\specialcell{}}& \multicolumn{2}{c ||}{\specialcell{\bf Basic}}& \multicolumn{2}{c ||}{\specialcell{\bf Typical}}& \multicolumn{2}{c}{\specialcell{\bf Adversarial}}  \\
       \hline & & & & & & \\
       \multicolumn{1}{c ||}{\specialcell{}}&  \multicolumn{1}{c |}{\specialcell{\bf AP ($\%$)}} & \multicolumn{1}{c ||}{\specialcell{\bf Success Rate ($\%$)}}&  \multicolumn{1}{c |}{\specialcell{\bf AP ($\%$)}} & \multicolumn{1}{c ||}{\specialcell{\bf Success Rate ($\%$)}}&  \multicolumn{1}{c |}{\specialcell{\bf AP ($\%$)}} & \multicolumn{1}{c}{\specialcell{\bf Success Rate ($\%$)}} \\
       \hline & & & & & & \\
       {\bf Planarity} & 81 & 74 & 69 & 67 & 48 & 47 \\
       {\bf Centroid} & 89 & 92 & 80 & 78 & 47 & 38 \\
       {\bf Planarity-Centroid} & 98 &  94 & 94 & {\bf 86} & 64 & 62 \\
       {\bf GQ-CNN (ADV)} & 83 & 77 & 75 & 67 & {\bf 86} & {\bf 81} \\
       {\bf GQ-CNN (DN3)} &{\bf 99} & {\bf 98} & {\bf 97} & 82 & 61 & 58 
        \end{tabular}
        }
        \caption{Performance of image-based grasping policies for 125 trials each on the Basic and Typical datasets and 100 trials each on the Adversarial datasets. We see that the GQ-CNN trained on Dex-Net 3.0 has the highest average precision on the Basic and Typical objects but has lower precision on the adversarial objects, which are very different than common objects in the training dataset. A GQ-CNN trained on the Adversarial dataset significantly outperforms all methods on these objects, suggesting that our model is able to capture complex geometries when the training dataset contains a large proportion of such objects. }
		\tablabel{policy-results}
%\vspace*{-20pt}
\end{table*}

\begin{figure*}[t!]
\centering
\includegraphics[scale=0.6]{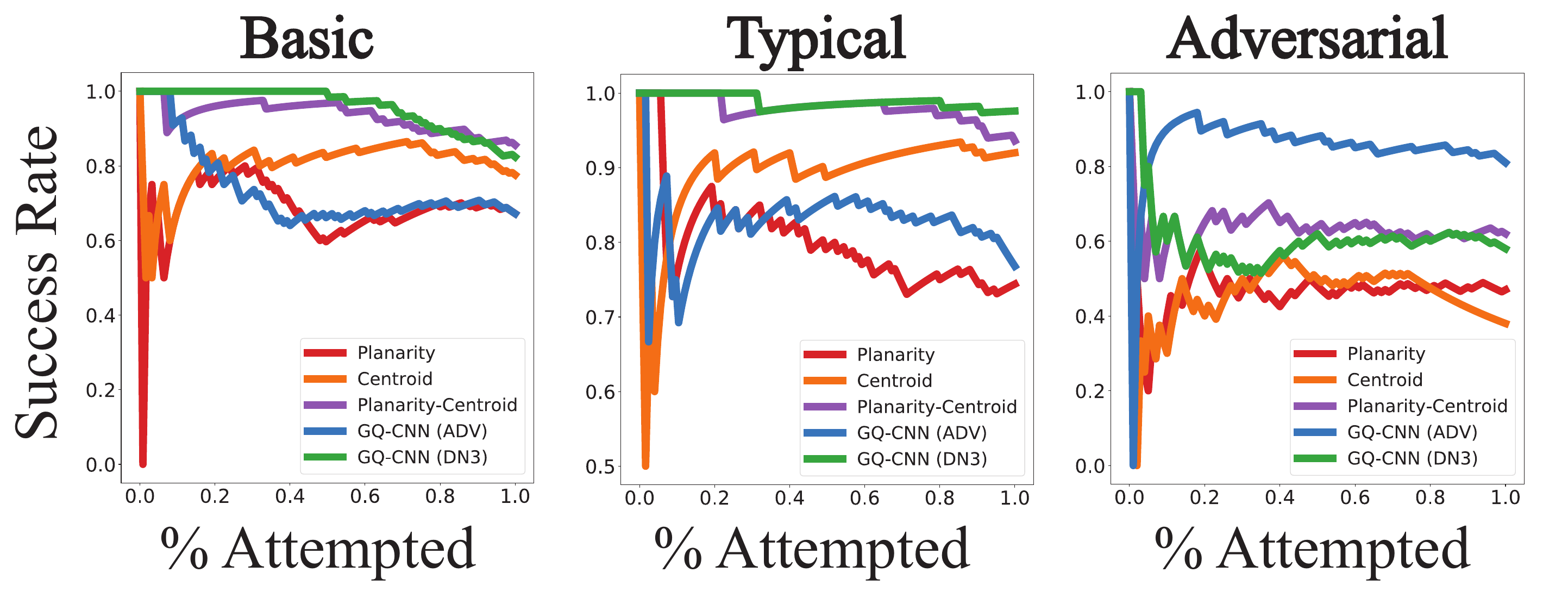}
\caption{Success vs Attempt Rate for 125 trials on each of the Basic and Typical object datasets and 100 trials each on the Adversarial object dataset. The GQ-CNN trained on Dex-Net 3.0 has near 100$\%$ precision on the Basic and Typical classes for a significant portion of attempts, suggesting that the GQ-CNN is able to predict when it is likely to fail on novel objects. The GQ-CNN trained on the Adversarial objects has a significantly higher precision on the Adversarial class but does not perform as well on the other objects.}
\figlabel{hyp2}
\end{figure*}

\subsection{Classification Performance on Known Objects}
To assess performance of our robustness metric on classifying grasps as successful or unsuccessful, we evaluated whether or not the metric was able to classify a set of grasps sampled randomly from the 3D object surface using the known 3D object geometry and pose of the Adversarial objects.
First, we sampled a set of $1000$ grasps uniformly at random from the surface of the 3D object meshes.
Then robot was presented one of the five Adversarial objects in a known stable pose, selected from the top three most probable stable poses.
We hand-aligned the object to a template image generated by rendering the object in a known pose on the table.
Then, the robot executed a grasp uniformly at random from the set of reachable grasps for the given stable pose.
In total, there were $600$ trials, $125$ per object.

We compared the predictions made for those grasps by the following metrics:
\begin{enumerate}
\item {\bf Planarity-Centroid (PC3D).} The inverse distance to the object centroid for sufficiently planar patches on the 3D object surface.
\item {\bf Spring Stretch (SS).} The maximum stretch among virtual springs in the suction contact model.
\item {\bf Wrench Resistance (WR).}
\item {\bf Robust Wrench Resistance (RWR).}
\end{enumerate}

We measured the Average Precision (AP), classification accuracy, and Spearman's rank correlation coefficient (which measures the correlation between the ranking of the metric value and successes on the physical robot).
\tabref{random-results} details the performance of each metric and the Precision-Recall curve is plotted in \figref{hyp1-pr}.
We see that the robust wrench resistance metric with our compliant spring contact model has the highest average precision and correlation with successes on the physical robot.

\begin{table}[t]
\centering
\resizebox{\textwidth}{!}{%
        \begin{tabular}{r || R{1.0cm} | R{1.0cm} | R{1.0cm}}
       \multicolumn{1}{c ||}{\specialcell{{\bf Metric}}}&  \multicolumn{1}{c |}{\specialcell{\bf AP ($\%$)}} & \multicolumn{1}{c}{\specialcell{\bf Accuracy ($\%$)}} & \multicolumn{1}{c}{\specialcell{\bf Rank Correlation}}\\
       \hline & & & \\
       PC3D & 71 &  68 & 0.36  \\
       SS & 75 &  74 & 0.49 \\
       WR & 78 & {\bf 77} & 0.52  \\
       RWR & {\bf 80} & 75 & {\bf 0.62}
        \end{tabular}
        }
		\caption{Performance of classification and correlation with successful object lifts and transports for various metrics of grasp quality based on 3D object meshes. The metrics SS, WR, and RWR all use our compliant suction contact model, and RWR uses our entire proposed method: checking seal formation, analyzing wrench resistance using the suction ring model, and computing robustness with Monte-Carlo sampling.}
		\tablabel{random-results}
\end{table}

\begin{figure}[t!]
\centering
\includegraphics[scale=0.7]{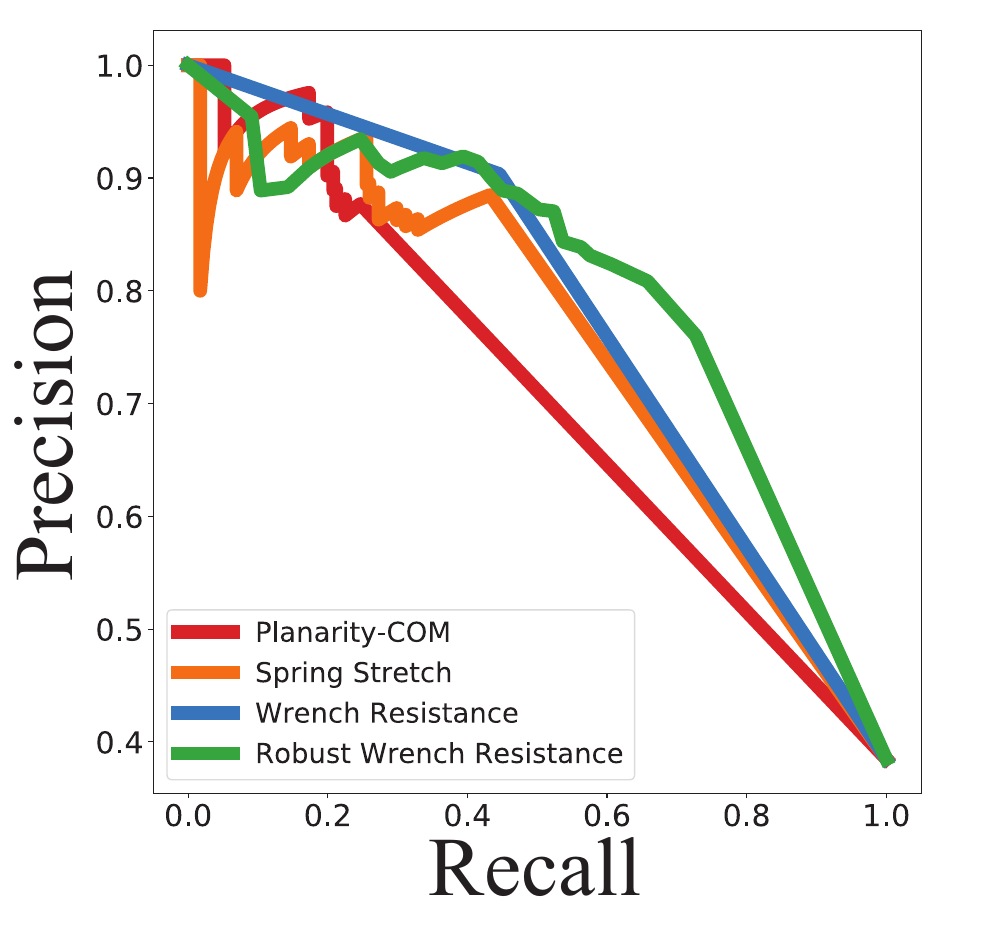}
\caption{Precision-Recall curve for classifying successful object lifts and transports using various metrics of grasp quality based on 3D object meshes.
\vspace*{-20pt}}
\figlabel{hyp1-pr}
\end{figure}

\subsection{Failure Modes}
Our system was not able to handle many objects due to material properties.
We broke up the failure objects into two categories:
\begin{enumerate}
	\item {\bf Imperceptible Objects:} Objects with (a) surface variations less than the spatial resolution of our Primesense Carmine 1.09 depth camera or (b) specularities or  transparencies that prevent the depth camera from sensing the object geometry. Thus the point-cloud-based grasping policies were not able to distinguish successes from failures.
	\item {\bf Impossible Objects:} Objects for which a seal cannot be formed either because objects are (a) non-porous or (b) lack a surface patch for which the suction cup can achieve a seal due to size or texture.
\end{enumerate}
\noindent These objects are illustrated in \figref{failure-datasets}.

\begin{figure}[t!]
\centering
\includegraphics[scale=0.22]{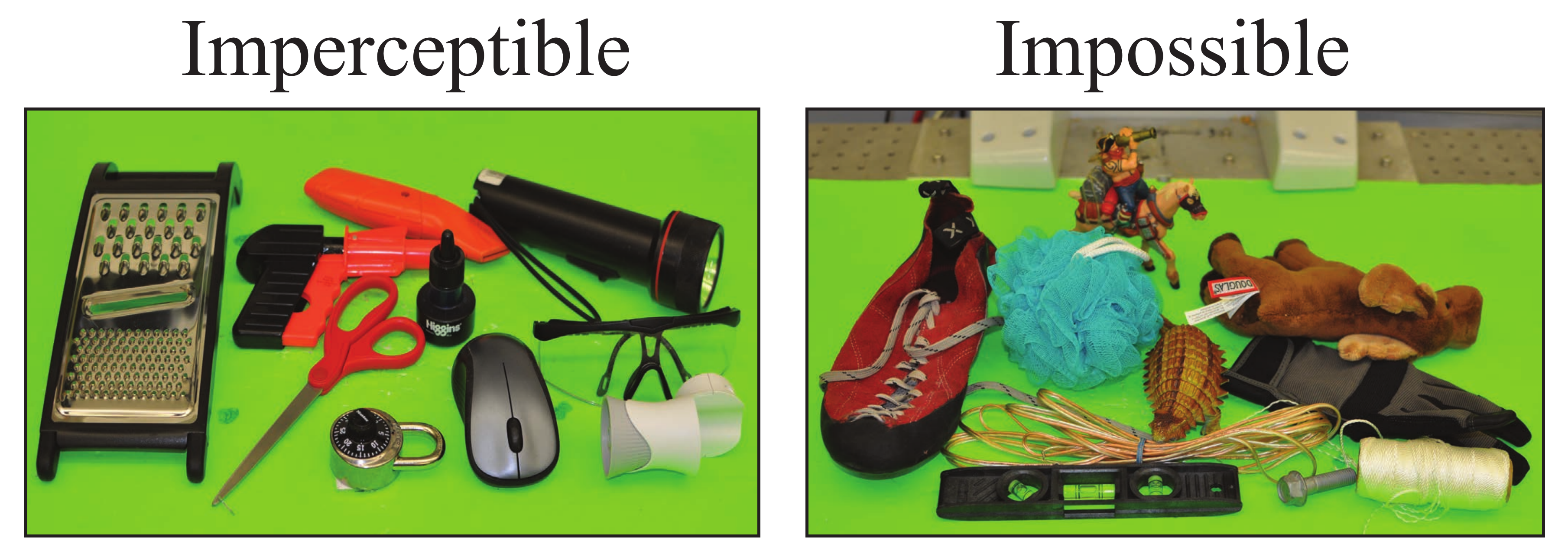}
\caption{Two categories that cannot be handled by any of the point-cloud based suction grasping policies. {\bf(Left)} Imperceptible objects, which cannot be handled by the system due to small surface variations that cannot be detected by the low-resolution depth sensor but do prevent seal formation. {\bf(Left)} Impossible objects, which cannot be handled by the system due to non-porosity or lack of an available surface to form a seal. }
\figlabel{failure-datasets}
\end{figure}

\section{Details of Quasi-Static Spring Seal Formation Model}
\seclabel{config-details}

In this section, we derive a detailed process for statically determining a final configuration of $C$ that achieves a complete seal against mesh $M$.
We assume that we are given a line of approach $\ell$ parameterized by $\mathbf{p}$, a target point on the surface of $M$, and $\mathbf{v}$, a vector pointing towards $P$ along the line of approach.

First, we choose an initial, undeformed configuration of $C$.
In this undeformed configuration of $C$, all of the springs of $C$ are in their resting positions, which means that the structural springs of $C$ form a right pyramid with a regular $n$-gon as its base.
This perfectly constrains the relative positions of the vertices of $C$, so all that remains is specifying the position and orientation of $C$ relative to the world frame.

We further constrain the position and orientation of $C$ such that $\ell$ passes through $\mathbf{a}$ and is orthogonal to the plane containing the base of $C$.
This leaves only the position of $\mathbf{a}$ and a rotation about $\ell$ as degrees of freedom.
For our purposes, the position of $\mathbf{a}$ along $\ell$ does not matter so long as $C$ is not in collision with $M$ and the base of $C$ is closer to $M$ than the apex is.
In general, we choose $\mathbf{a}$ such that $\lVert \mathbf{p} - \mathbf{a} \rVert > x + h$, where $x$ is the largest extent of the object's vertices.
For the rotation about $\ell$, we simply select a random initial angle.
Theoretically, the rotation could affect the outcome of our metric, but as long as $n$ is chosen to be sufficiently large, the result is not sensitive to the chosen rotation angle.

Next, given the initial configuration of $C$, we compute the final locations of the perimeter springs on the surface of $M$ under two main constraints:
\begin{itemize}
    \item The perimeter springs of $C$ must not deviate from their initial locations when projected back onto the plane containing the base of $C$'s initial right pyramid.
    \item The perimeter springs of $C$ must lie flush against the mesh $M$.
\end{itemize}
Essentially, this means that the perimeter springs will lie on the intersection of $M$ with a right prism $K$ whose base is the base of the initial configuration's right pyramid and whose height is sufficient such that $K$ passes all the way through $M$.
The base vertices of $C$ will lie at the intersection of $M$ and $K$'s side edges, and the perimeter springs of $C$ will lie along the intersection of $M$ and $K$'s side faces.

Finally, given a complete configuration of the perimeter vertices of $C$ as well as the paths of the perimeter springs along the surface of $M$, we compute the final location of the cup apex $\mathbf{a}$.
We work with three main constraints:
\begin{itemize}
    \item $\mathbf{a}$ must lie on $\ell$.
    \item $\mathbf{a}$ must not be below the surface of $M$ (i.e. $\mathbf{v}^T(\mathbf{a} - \mathbf{p}) \leq 0$).
    \item $\mathbf{a}$ should be chosen such that the average displacement between $\mathbf{a}$ and the perimeter vertices along $\mathbf{v}$ remains equal to $h$.
\end{itemize}
Let $\mathbf{a}^* = \mathbf{p} - t^*\mathbf{v}$. Then, the solution distance $t^*$ is given by
\begin{equation*}
    t^* = \min\Bigg(\Big[\frac{1}{n} \sum_{i=1}^n (\mathbf{v}_i - \mathbf{p})^T \mathbf{v}\Big] - h, 0\Bigg).
\end{equation*}

When thresholding the energy in each spring, we use a per-spring threshold of a $10\%$ change in length, which was used as the spring stretch limit in ~\cite{provot1995deformation}.

\section{Suction Contact Model}
\seclabel{contact-deriv}

The basis of contact wrenches for the suction ring model is illustrated in \figref{model}.
The contact wrenches are not independent due to the coupling of normal force and friction, and they may be bounded due to material properties.
In this section we prove that wrench resistance can be computed with quadratic programming, we derive constraints between the contact wrenches in the suction ring model, and we explain the limits of the soft finger suction contact models for a single suction contact.

\subsection{Computing Wrench Resistance with Quadratic Programming}
The object wrench set for a grasp using a contact model with $m$ basis wrenches is $\Lambda = \{\bw \in \bR^6 \mid \bw =  G \alpha \text{ for some } \alpha \in \mF\}$, where $G \in \bR^{6 \times m}$ is a set of $m$ basis wrenches in the object coordinate frame, and $\mF \subseteq \bR^m$ is a set of constraints on contact wrench magnitudes~\cite{murray1994mathematical}.
The grasp map $G$ can be decomposed as $G = A W$ where $A \in \bR^{6\times6}$ is the {\it adjoint transformation} mapping wrenches from the contact to the object coordinate frame and $W \in {6 \times m}$ is the {\it contact wrench basis}, a set of $m$ orthonormal basis wrenches in the contact coordinate frame~\cite{murray1994mathematical}.
\begin{definition}
A grasp $\bu$ achieves {\it wrench resistance} with respect to $\bw$ if $-\bw \in \Lambda$.
\end{definition}

\begin{proposition}
Let $G$ be the grasp map for a grasp $\bu$.
Furthermore, let $\epsilon^* = \text{argmin}_{\alpha \in \mF}\|G \alpha + \bw\|_2^2$.
Then $\bu$ can resist $\bw$ iff $\epsilon^* = 0$.
\end{proposition}
\begin{proof}
($\Rightarrow$).
Assume $\bu$ can resist $\bw$.
Then $-\bw \in \Lambda$ and therefore $\exists \alpha \in \mF$ such that $G \alpha =  -\bw \Rightarrow G \alpha + \bw = \mathbf{0}$.
($\Leftarrow$).
Assume $\epsilon^* = 0$.
Then $\exists \alpha \in \mF$ such that $G \alpha + \bw = \mathbf{0} \Rightarrow G \alpha = -\bw \Rightarrow -\bw \in \Lambda$.
\end{proof}
When the set of admissible contact wrench magnitudes $\mF$ is defined by linear equality and inequality constraints, the $\text{min}_{\alpha \in \mF}\|G \alpha + \bw\|_2^2$, is a Quadratic Program which can be solved exactly by modern solvers.

\subsection{Derivation of Suction Ring Contact Model Constraints}
\begin{figure}[t!]
\centering
\includegraphics[width=0.9\textwidth]{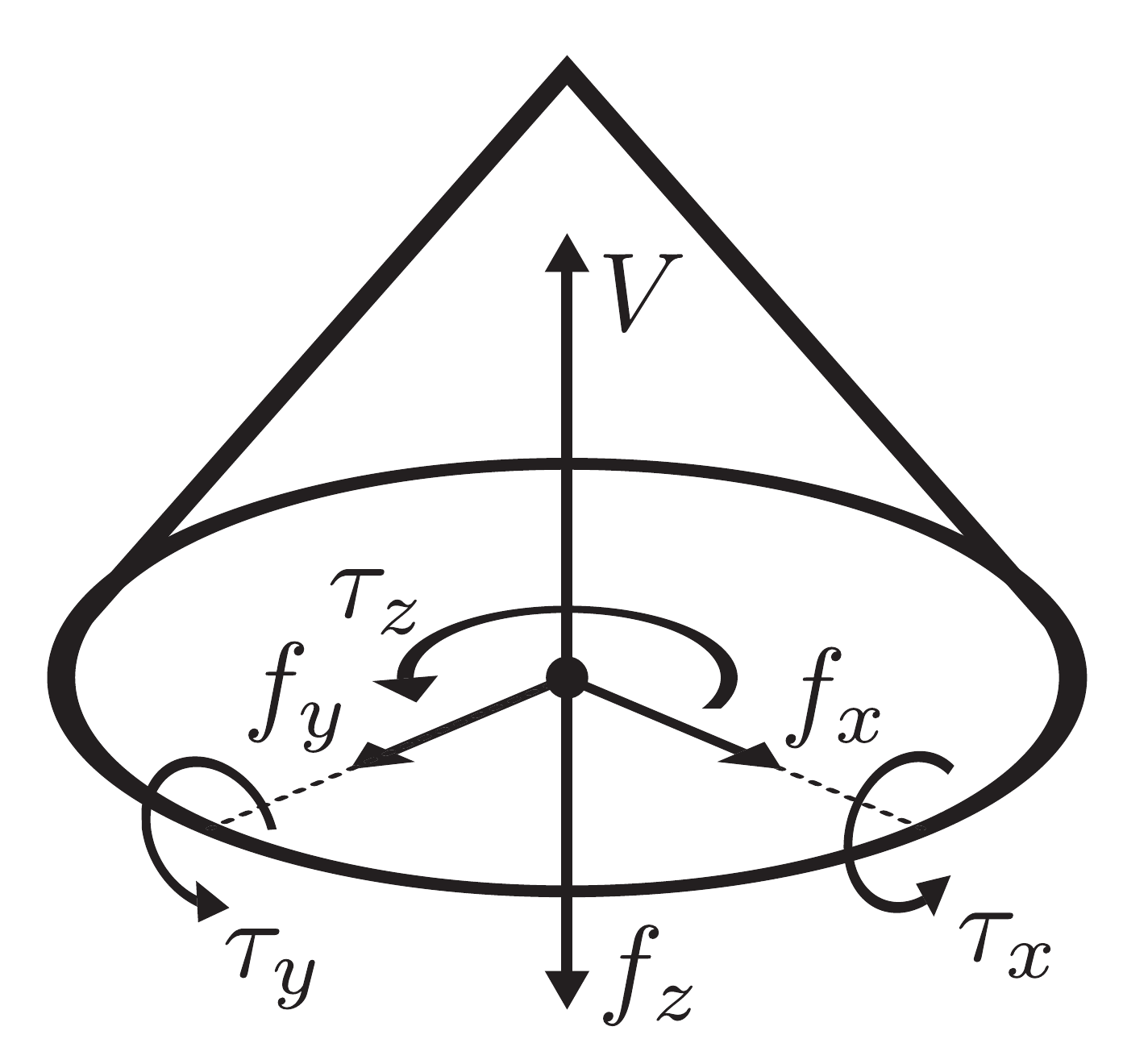}
\caption{Wrench basis for the compliant suction ring contact model. The contact exerts a constant pulling force on the object of magnitude $V$ and additionally can push or pull the object along the contact $z$ axis with force $f_z$. The suction cup material exerts a normal force $f_N = f_z + V$ on the object through a linear pressure distribution on the ring. This pressure distribution induces a friction limit surface bounding the set of possible frictional forces in the tangent plane $f_t = (f_x, f_y)$ and the torsional moment $\tau_z$, and also induces torques $\tau_x$ and $\tau_y$ about the contact $x$ and $y$ axes due to elastic restoring forces in the suction cup material.
\vspace*{-20pt}}
\figlabel{model}
\end{figure}

Our suction contact model assumes the following:
\begin{enumerate}
	\item Quasi-static physics (e.g. inertial terms are negligible).
	\item The suction cup contacts the object along a circle of radius $r$ (or ``ring") in the $xy-$plane of the contact coordinate frame.
	\item The suction cup material behaves as a ring of infinitesimal springs per unit length. Specifically, we assume that the pressure along the $z-$axis in contact coordinate frame satisfies $p(\theta) = k \delta_z(\theta)$ where $\delta_z$ is displacement along the $z$-axis and $k \in \bR$ is a spring constant (per unit length). The cup does not permit deformations along the $x$ or $y$ axes.
	\item The suction cup material is well approximated by a spring-mass system. Furthermore, points on the contact ring are in static equilibrium with a linear displacement along the $z-$axis from the equilibrium position: $\delta_z(\theta) = \delta_0 + a r \cos(\theta) + b r \sin(\theta)$. Together with Assumption 3, this implies that:
	 \begin{align*}
	 p(\theta) = p_0 + p_x \cos(\theta) + p_y \sin(\theta)
	\end{align*}	
	\noindent for real numbers $p_0, p_x,$ and $p_y$. 
	\item The force on the object due to the vacuum is a constant $-V$ along the $z-$ axis of the contact coordinate frame.
	\item The object exerts a normal force on the object $f_N = f_z + V$ where $f_z$ is the force due to actuation. This assumption holds when analyzing the ability to resist disturbing wrenches because the material can apply passive forces but may not hold when considering target wrenches that can be actuated.
	%\item The suction cup material exerts a constant normal force on the object to exactly balance the force due to suction: $f_{N} = V$.
	%\item The pressure between the suction contact ring and object is linear. At angle $\theta$, $p(\theta) = p_0 + a r cos(\theta) + b r sin(\theta)$.
\end{enumerate}

The magnitudes of the contact wrenches are constrained due to (a) the friction limit surface~\cite{kao2008contact}, (b) limits on the elastic behavior of the suction cup material, and (c) limits on the vacuum force.

\subsubsection{Friction Limit Surface}
The values of the tangential and torsional friction are coupled through the planar external wrench and thus are jointly constrained.
This constraint is known as the {\it friction limit surface}~\cite{kao2008contact}.
We can approximate the friction limit surface by computing the maximum friction force and torsional moment under a pure translation and a pure rotation about the contact origin.

The tangential forces have maximum magnitude under a purely translational disturbing wrench with unit vector $\hat{\bv}$ in the direction of the velocity:
\begin{align*}
	f_x &\leq \int \limits_{0}^{2 \pi} \mu  \hat{\bv}_x p(\theta) d \theta \\
    &\leq \int \limits_{0}^{2 \pi} \mu  \hat{\bv}_x \left(p_0 + p_x \cos(\theta) + p_y \sin(\theta)\right) d \theta \\
    &\leq 2 \pi \mu \hat{\bv}_x p_0 \\
    f_y &\leq 2 \pi \mu \hat{\bv}_y p_0 \\
    \|f_f\|_2^2 &\leq (2 \pi \mu \hat{\bv}_x p_0)^2 + (2 \pi \mu \hat{\bv}_y p_0)^2 \\
    &= (2 \pi \mu p_0)^2 \left( \hat{\bv}_x^2 + \hat{\bv}_y^2 \right) \\
    &= (2 \pi \mu p_0)^2 \\
    &= \mu f_N
\end{align*}

The torsional moment has a maximum moment under a purely rotational disturbing wrench about the contact $z-$ axis.
This disturbing wrench can be described with a unit vector $\hat{\bv}(\theta) = (\sin(\theta), -\cos(\theta), 0)$.
Thus the torsional moment is bounded by:
\begin{align*}
	|\tau_z| &\leq \int \limits_{0}^{2 \pi} \mu r p(\theta) d \theta = \int \limits_{0}^{2 \pi} \mu r \left(p_0 + p_x \cos(\theta) + p_y \sin(\theta)\right) d \theta \\
	&\leq 2 \pi \mu r p_0 \\
	&\leq r \mu f_N
\end{align*}

We can approximate the friction limit surface by the ellipsoid~\cite{kao1992quasistatic, kao2008contact}:
\begin{align*}
	\frac{\|f_t\|_2^2}{(\mu f_N)^2} + \frac{ |\tau_z|^2}{(r \mu f_N)^2} \leq 1
\end{align*}
While this constraint is convex, in practice many solvers for Quadratically Constrained Quadratic Programs (QCQPs) assume nonconvexity.
We can turn this into a linear constraint by bounding tangential forces and torsional moments in a rectangular prism inscribed within the ellipsoid:
\begin{align*}
	|f_x| \leq \frac{\sqrt{3}}{3} \mu f_N \\
	|f_y| \leq \frac{\sqrt{3}}{3} \mu f_N \\
	|\tau_z| \leq \frac{\sqrt{3}}{3} r \mu f_N \\
\end{align*}

\subsubsection{Elastic Restoring Torques}
The torques about the $x$ and $y$ axes are also bounded.
Let $\bw(\theta) = (r \cos(\theta), r \sin(\theta), 0)$.
Then:
\begin{align*}
	\tau_t &= \int \limits_{0}^{2 \pi} (\bw(\theta) \times \be_z) p(\theta) d \theta \\
	\tau_x &= \int \limits_{0}^{2 \pi} r \sin(\theta) p(\theta) d \theta \\
	&= \int \limits_{0}^{2 \pi} r \sin(\theta) \left(p_0 + p_x \cos(\theta) + p_y \sin(\theta)\right) d \theta \\
	&= \int \limits_{0}^{2 \pi} r p_y \sin^2(\theta) d \theta \\
	&= \pi r p_y \\
	\tau_y &= \pi r p_x \\
	\|\tau_{e}\|_2^2 &= \pi^2 r^2 (p_x^2 + p_y^2) \\
	&\leq \pi^2 r^2 \kappa^2
\end{align*}
\noindent where $\kappa$ is the {\it elastic limit} or {\it yield strength} of the suction cup material, defined as the stress at which the material begins to deform plastically instead of linearly.

\subsubsection{Vacuum Limits}
The ring contact can exert forces $f_z$ on the object along the $z$ axis through motor torques that transmit forces to the object through the ring of the suction cup.
Under these assumptions, the normal force exerted on the object by the suction cup material is:
\begin{align*}
	f_N &= \int \limits_{0}^{2 \pi} p(\theta) d \theta = \int \limits_{0}^{2 \pi} \left(p_0 + p_x \cos(\theta) + p_y \sin(\theta)\right) d \theta \\ 
	&= 2 \pi p_0 
\end{align*}
\noindent Note also that $f_N = f_z + V$, where $f_z$ is the $z$ component of force on the object, since the normal force must offset the force due to vacuum $V$ even when no force is being applied on the object.

\subsubsection{Constraint Set}
Taking all constraints into account, we can describe $\mF$ with a set of linear constraints:
\begin{align*}
	\text{{\bf Friction:}} && \sqrt{3} | f_x | &\leq \mu f_N & \sqrt{3} | f_y | &\leq \mu f_N & \sqrt{3} | \tau_z| &\leq r \mu f_N  \\
	\text{{\bf Material:}} && \sqrt{2}| \tau_x | &\leq \pi  r \kappa & \sqrt{2} | \tau_y | &\leq \pi  r \kappa \\
	\text{{\bf Suction:}} &&  f_z &\geq -V
\end{align*}
\noindent Since these constraints are linear, we can solve for wrench resistance in the our contact model using Quadratic Programming.
In this paper we set $V=250N$ and $\kappa=0.005$.

%By Proposition 5.2 of~\cite{murray1994mathematical}, $G$ must be surjective to achieve force closure.
%However, $rank(G) \leq \text{min}(rank(A), rank(W)) < 6$ since $rank(W) = 4$.
%Therefore $G$ is not surjective.

%The grasp map $G$ can be decomposed as $G = A W$ where $A \in \bR^{6\times6}$ is the {\it adjoint transformation} mapping wrenches from the contact to the object coordinate frame and $W \in {6 \times m}$ is the {\it contact wrench basis}, a set of $m$ orthonormal basis wrenches in the contact coordinate frame~\cite{murray1994mathematical}.
%In this paper, we use an object coordinate frame centered on the object center-of-mass and a contact coordinate frame that is centered on the point of contact $\rho \in \bR^3$ and oriented such that the $z$axis aligns with the outward-pointing surface normal $\bn$ at $\rho$ and the $x-$ and $y-$ axes are an orthogonal basis to the tangent plane at $\rho$ spanned by the orthonormal tangent vectors $\bt_x$ and $\bt_y$.

%A common model for suction contact 
%The most common suction contact model in the literature~\cite{kolluru1998modeling, mantriota2007theoretical, stuart2015suction, valencia20173d, yoshida2010design}

\subsection{Limits of the Soft Finger Suction Contact Model}
The most common suction contact model in the literature~\cite{kolluru1998modeling, mantriota2007theoretical, stuart2015suction, valencia20173d, yoshida2010design} considers normal forces from motor torques, suction forces from the pressure differential between inside the cup and the air outside the object, and both tangential and torsional friction resulting from the contact area between the cup and the object.
Let $\be_x$, $\be_y$, and $\be_z$ be unit basis vectors along the $x$, $y$, and $z$ axes.
The contact model is specified by:
\begin{align*}
	W &= \left[ \begin{array}{cccc}
	\be_x & \be_y & \be_z & \bzero \\
	\bzero & \bzero & \bzero & \be_z
	\end{array}
	\right] \\
	\alpha &= (f_x, f_y, f_z, \tau_z) \in \mF \text{ if and only if:} \\
	&\sqrt{f_x^2 + f_y^2} \leq \mu f_z \\
	&| \tau_z | \leq \gamma | f_z |
\end{align*}
\noindent The first constraint enforces Coulomb friction with coefficient $\mu$.
%The second constraint bounds the normal force by the maximum vacuum force $f_v$, which can only pull out of the object, and the maximum motor force $f_m$, which can only push into the object along the inward-pointing surface normal.
The second constraint ensures that the net torsion is bounded by the normal force, since torsion results from the net frictional moment from a contact area.
Unlike contact models for rigid multifinger grasping, $f_z$ can be positive or negative due to the pulling force of suction.

\begin{proposition}
Under the soft suction contact model, a grasp with a single contact point cannot resist torques about axes in the contact tangent plane.
\end{proposition}
\begin{proof}
%By Proposition 5.2 of~\cite{murray1994mathematical}, $G$ must be surjective to achieve force closure.
%However, $rank(G) \leq \text{min}(rank(A), rank(W)) < 6$ since $rank(W) = 4$.
%Therefore $G$ is not surjective.
The wrench $\bw = (\bzero, \tau_e)$ is not in the range of $W$ because it is orthogonal to every basis wrench (column of $W$).
\end{proof}

The null space of $W$ is spanned by the wrenches $\bw_1 = (\bzero, \be_x)$ and $\bw_2 = (\bzero, \be_y)$, suggesting that a single suction contact cannot resist torques in the tangent plane at the contact.
This defies our intuition since empirical evidence suggests that a single point of suction can reliably hold and transport objects to a receptacle in applications such as the Amazon Picking Challenge~\cite{eppner2016lessons, hernandez2016team}.

\section{GQ-CNN Performance}
\seclabel{training}
The GQ-CNN trained on Dex-Net 3.0 had an accuracy of 93.5$\%$ on a held out validation set of approximately 552,000 datapoints.
\figref{roc-conv} shows the precision-recall curve for the GQ-CNN validation set and the optimized 64 Conv1\_1 filters, each of which is 7$\times$7.
\figref{policy-examples} illustrates the probability of success predicted by the GQ-CNN on candidates grasps from several real point clouds.

\begin{figure*}[t!]
\centering
\includegraphics[scale=0.6]{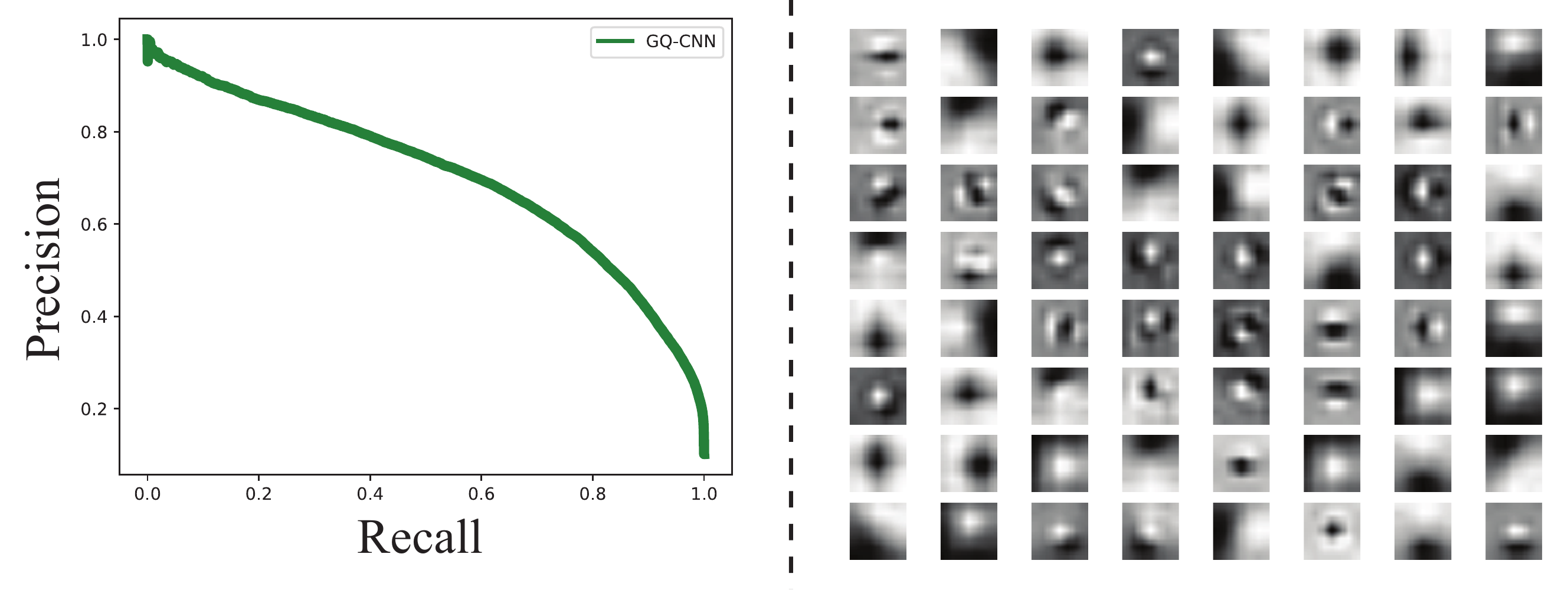}
\caption{{\bf(Left)} Precision-recall curve for the GQ-CNN trained on Dex-Net 3.0 on the validation set of 552,000 pairs of grasps and images. {\bf(Right)} The 64 Conv1\_1 filters of the GQ-CNN. Each is 7$\times$7. We see that the network learns circular filters which may be used to assess the surface curvature about the ring of contact between the suction cup and object.}
\figlabel{roc-conv}
\end{figure*}

\begin{figure}[t!]
\centering
\includegraphics[scale=0.30]{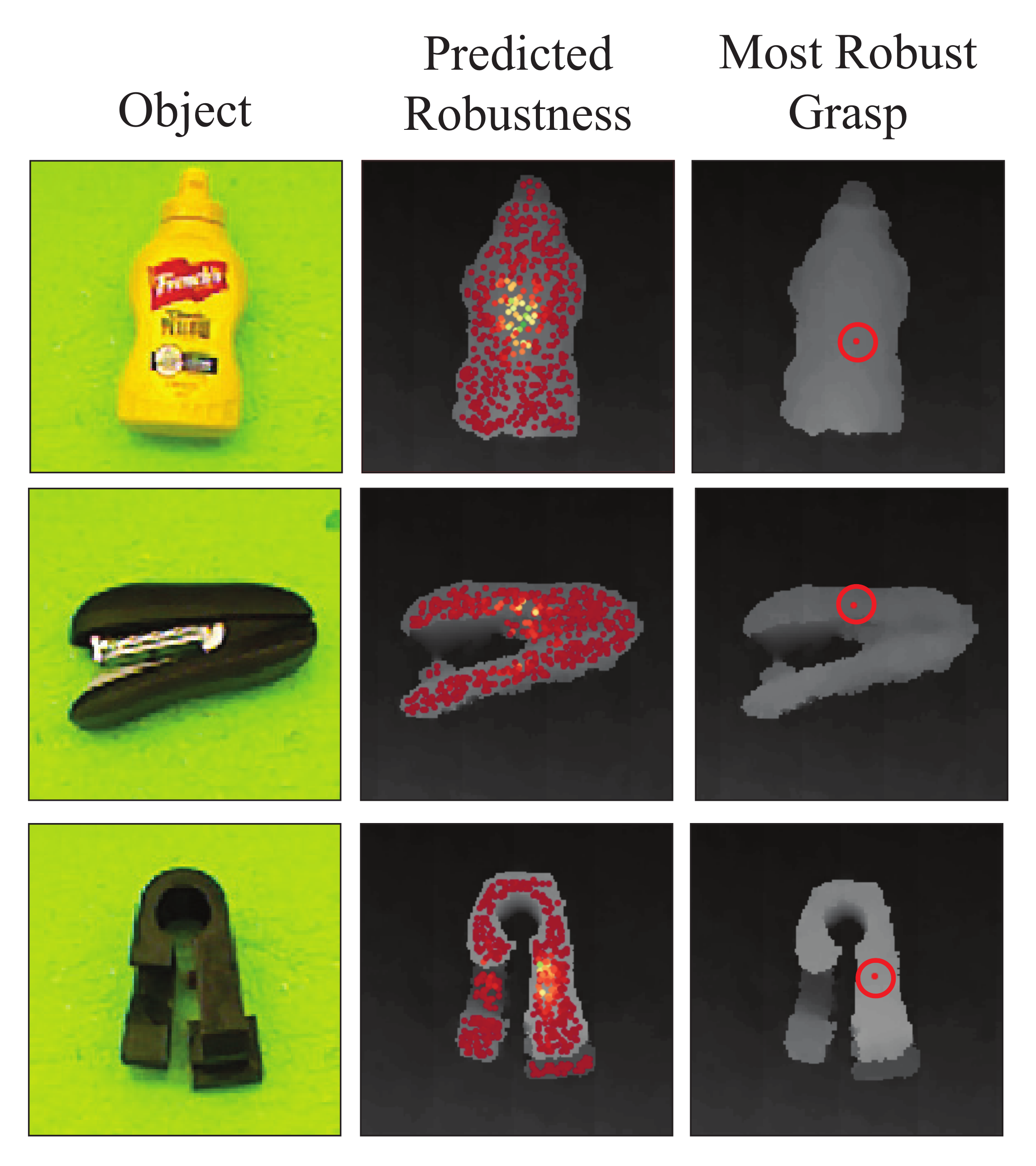}
\caption{Robust grasps planned with the Dex-Net 3.0 GQ-CNN-based policy on example RGB-D point clouds. \textbf{(Left)} The robot is presented an object in isolation. \textbf{(Middle)} Initial candidate suction target points colored by the predicted probability of success from zero (red) to one (green). Robust grasps tend to concentrate around the object centroid. \textbf{(Right)} The policy optimizes for the grasp with the highest probability of success using the Cross Entropy Method.
\vspace*{-20pt}}
\figlabel{policy-examples}
\end{figure}
\section{Environment Model}

\begin{figure*}[t!]
\centering
\includegraphics[width=0.8\textwidth]{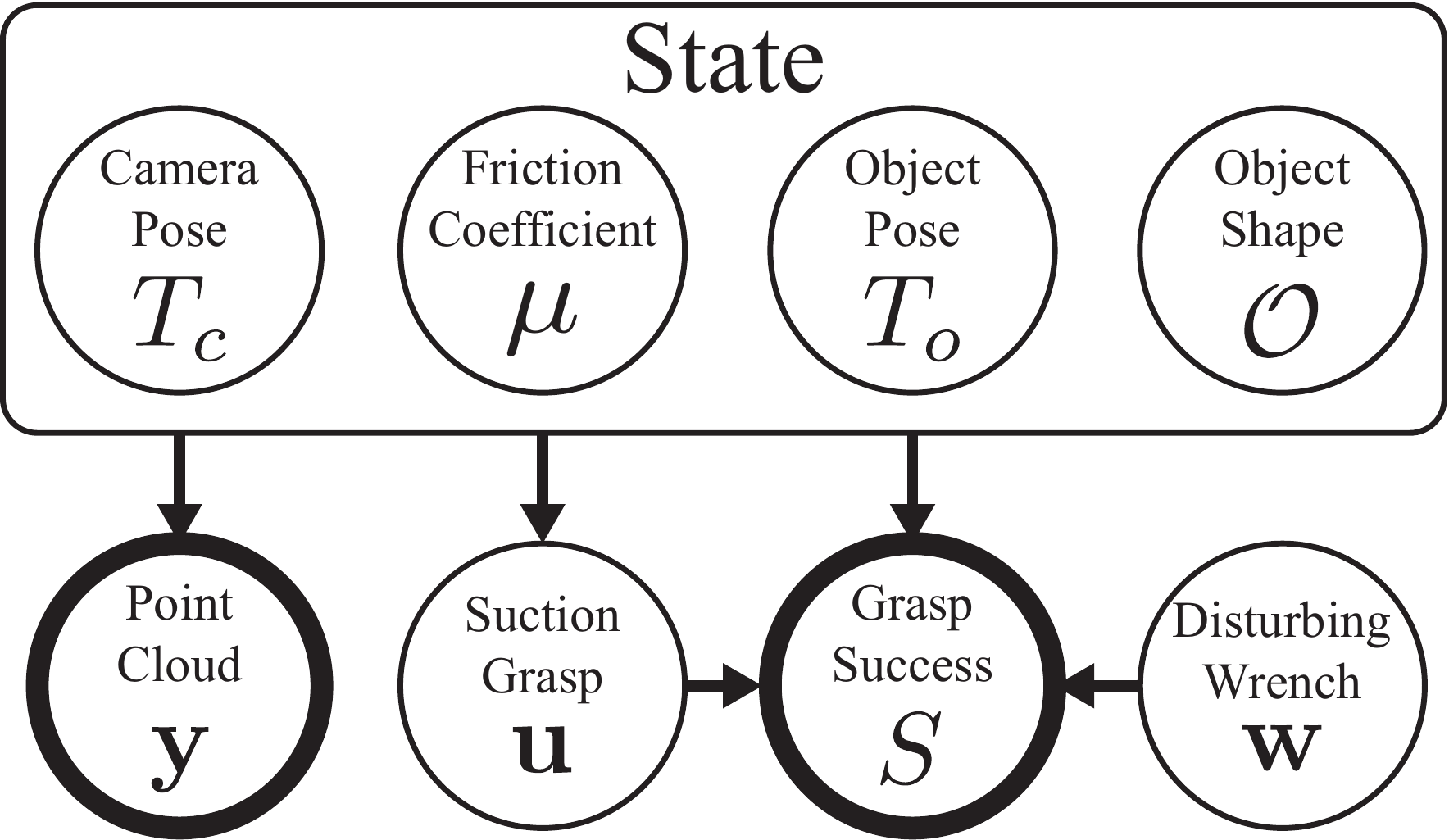}
\caption{A probabilistic graphical model of the relationship between the ability to resist external wrenches e.g. due to gravity under perturbations in object pose, gripper pose, camera pose, and friction.}
\figlabel{graphical-model}
\end{figure*}

To learn to predict grasp robustness based on noisy point clouds, we generate the Dex-Net 3.0 training dataset of point clouds, grasps, and grasp success labels by sampling tuples $(S_i, \bu_i, \by_i)$ from a joint distribution $p(S, \bu, \bx, \by)$ that is composed of distributions on:
\begin{itemize}
	\item {\bf States: $p(\bx)$:} A prior on possible objects, object poses, and camera poses that the robot will encounter.
	\item {\bf Grasp Candidates: $p(\bu | \bx)$:} A prior constraining grasp candidates to target points on the object surface.
	\item {\bf Grasp Successes $p(S \mid \bu, \bx)$:} A stochastic model of wrench resistance for the gravity wrench.
	\item {\bf Observations $p(\by \mid \bx)$:} A sensor noise model.
\end{itemize}
\noindent Our graphical model is illustrated in \figref{graphical-model}.

\begin{table}[t]
\centering
\resizebox{\textwidth}{!}{%
   \begin{tabular}{| r | c |}
   \hline
  {\bf Distribution} & {\bf Description}\\ \hline
  $p(\mu)$ & truncated Gaussian distribution over friction coefficients \\
  \hline
  $p(\mO)$ & discrete uniform distribution over 3D object models \\
  \hline
  $p(T_{o} | \mO)$ & \begin{tabular}[x]{@{}c@{}} continuous uniform distribution over the discrete set of \\ object stable poses and planar poses on the table surface \end{tabular}  \\
  \hline
  $p(T_{c})$ & \begin{tabular}[x]{@{}c@{}} continuous uniform distribution over spherical coordinates \\ for radial bounds $[r_{\ell}, r_{u}]$ and polar angle in $[0, \delta]$ \end{tabular} \\
  \hline
  \end{tabular}}
        \caption{Details of the distributions used in the Dex-Net 2.0 graphical model for generating the Dex-Net training dataset.  }
		\tablabel{distributions}
\end{table}

\subsection{Details of Distributions}
We follow the state model of~\cite{mahler2017dex}, which we repeat here for convenience.
The parameters of the sampling distributions were set by maximizing average precision of the $Q$ values using grid search for a set of grasps attempted on an ABB YuMi robot on a set of known 3D printed objects (see \secref{additional-known-objects}).

We model the state distribution as $p(\bx) = p(\mu) p(\mO) p(T_{o} | \mO)  p(T_{c})$.
We model $p(\mu)$ as a Gaussian distribution $\mN(0.5, 0.1)$ truncated to $[0,1]$.
We model $p(\mO)$ as a discrete uniform distribution over 3D objects in a given dataset.
We model $p(T_{o} | \mO) = p(T_{o} | T_{s}) p(T_{s} | \mO) $, where is $p(T_{s} | \mO)$ is a discrete uniform distribution over object stable poses and $p(T_{o} | T_{s})$ is uniform distribution over 2D poses: $\mU([-0.1,0.1]\times[-0.1,0.1]\times[0,2\pi))$.
We compute stable poses using the quasi-static algorithm given by Goldberg et al.~\cite{goldberg1999part}.
We model $p(T_{c})$ as a uniform distribution on spherical coordinates $r, \theta, \phi \sim \mU([0.5, 0.7]\times[0,2\pi)\times[0.01\pi, 0.1\pi])$, where the camera optical axis always intersects the center of the table.
The parameters of the sampling distributions were set by maximizing average precision of the $Q$ values using grid search for a set of grasps attempted on an ABB YuMi robot on a set of known 3D printed objects (see \secref{additional-known-objects}).

Our grasp candidate model $p(\bu \mid \bx)$ is a uniform distribution over points samples on the object surface, with the approach direction defined by the inward-facing surface normal at each point.

We follow the observation model of~\cite{mahler2017dex}, which we repeat here for convenience.
Our observation model $p(\by \mid \bx)$ model images as $\by = \alpha * \hat{\by} + \epsilon$ where $\hat{\by}$ is a rendered depth image created using OSMesa offscreen rendering.
We model $\alpha$ as a Gamma random variable with shape$=1000.0$ and scale=$0.001$.
We model $\epsilon$ as Gaussian Process noise drawn with measurement noise $\sigma=0.005$ and kernel bandwidth $\ell = \sqrt{2}px$.

Our grasp success model $p(S \mid \bu, \bx)$ specifies a distribution over wrench resistance due to perturbations in object pose, gripper pose, friction coefficient, and the disturbing wrench to resist.
Specifically, we model $p(S \mid \bu, \bx) = p(S \mid \hat{\bu}, \hat{\bx}, \bw) p(\hat{\bx} \mid \bx) p(\hat{\bu} \mid \bu) p(\bw)$.
We model $p(\bw)$ as the wrench exerted by gravity on the object center-of-mass with zero-mean Gaussian noise $\mN(\mathbf{0}_{3}, 0.01 \mathbf{I}_{3})$ assuming as mass of 1.0kg.
We model $p(\hat{\bu} \mid \bu)$ as a grasp perturbation distribution where the suction target point is perturbed by zero-mean Gaussian noise $\mN(\mathbf{0}_{3}, 0.001 \mathbf{I}_{3})$ and the approach direction is perturbed by zero-mean Gaussian noise in the rotational component of Lie algebra coordinates $\mN(\mathbf{0}_{3}, 0.1 \mathbf{I}_{3})$.
We model $p(\hat{\bx} \mid \bx)$ as a state perturbation distribution where the pose $T_o$ is perturbed by zero-mean Gaussian noise in  Lie algebra coordinates with translational component $\mN(\mathbf{0}_{3}, 0.001 \mathbf{I}_{3})$ and rotational component $\mN(\mathbf{0}_{3}, 0.1 \mathbf{I}_{3})$ and the object center of mass is perturbed by zero-mean Gaussian noise $\mN(\mathbf{0}_{3}, 0.0025 \mathbf{I}_{3})$.
We model $p(S \mid \hat{\bu}, \hat{\bx}, \bw)$ a Bernoulli with parameter 1 if $\hat{\bu}$ resists $\bw$ given the state $\hat{\bx}$ and parameter 0 if not.

\subsection{Implementation Details}
To efficiently implement sampling, we make several optimizations.
First, we precompute the set of grasps for every 3D object model in the database and take a fixed number of samples of grasp success from $p(S \mid \bu, \bx)$ using quadratic programming for wrench resistance evaluation.
We convert the samples to binary success labels by thresholding the sample mean by $\tau = 0.5$.
We also render a fixed number of depth images for each stable pose independently of grasp success evaluation.
Finally, we sample a set of candidate grasps from the object in each depth image and transform the image to generate a suction grasp thumbnail centered on the target point and oriented to align the approach axis with the middle column of pixels for GQ-CNN training.

\section*{Acknowledgments}
{\small
This research was performed at the AUTOLAB at UC Berkeley in affiliation with the Berkeley AI Research (BAIR) Lab, the Real-Time Intelligent Secure Execution (RISE) Lab, and the CITRIS ”People and Robots” (CPAR) Initiative.
The authors were supported in part by donations from Siemens, Google, Honda, Intel, Comcast, Cisco, Autodesk, Amazon Robotics, Toyota Research Institute, ABB, Samsung, Knapp, and Loccioni, Inc and by the Scalable Collaborative Human-Robot Learning (SCHooL) Project, NSF National Robotics Initiative Award 1734633.
Any opinions, findings, and conclusions or recommendations expressed in this material are those of the author(s) and do not necessarily reflect the views of the Sponsors.
We thank our colleagues who provided helpful feedback, code, and suggestions, in particular Ruzena Bajcsy, Oliver Brock, Peter Corke, Chris Correa, Ron Fearing, Roy Fox, Bernhard Guetl, Menglong Guo, Michael Laskey, Andrew Lee, Pusong Li, Jacky Liang, Sanjay Krishnan, Fritz Kuttler, Stephen McKinley, Juan Aparicio Ojea, Michael Peinhopf, Peter Puchwein, Alberto Rodriguez, Daniel Seita, Vishal Satish, and Shankar Sastry.
}

\addtolength{\textheight}{-3.5cm}

%\balance
\bibliographystyle{IEEEtranS}
\bibliography{bibliography}

\end{document}